\documentclass{article}

\usepackage{microtype}
\usepackage{graphicx}
\usepackage{subfigure}
\usepackage{booktabs} 
\usepackage{bbm}

\usepackage{hyperref}

\usepackage[accepted]{icml2024}

\usepackage{amsmath}
\usepackage{amssymb}
\usepackage{mathtools}
\usepackage{amsthm}

\usepackage{enumerate}
\usepackage{amsfonts, dsfont}

\def\I{{\mathbb I}}
\def\P{{\mathbb P}}
\def\pr{\mathbb{P}}

\def\L{{\mathcal L}}
\def\X{{\mathbf{X}}}
\def\x{{\mathbf{x}}}

\def\E{{\mathbb E}}

\def\T{{\mathcal T}}

\def\H{{\mathbf{H}}}
\let\hat\widehat

\newtheorem{thm}{Theorem}

\newtheorem{Lemma}{Lemma}

\newtheorem{Definition}{Definition}
\newtheorem{Remark}{Remark}

\newcommand{\Ncal}{\mathcal N}

\newcommand{\Xbf}{\mathbf X}
\newcommand{\xbf}{\mathbf x}

\newcommand{\codecomment}[1]{\textbf{\color{black}// #1}}
\def\I{{\mathbb I}}
\def\P{{\mathbb P}}
\def\pr{\mathbb{P}}

\def\H{{\mathbf{H}}}

\def\L{{\mathcal L}}
\def\X{{\mathbf{X}}}
\def\x{{\mathbf{x}}}

\def\E{{\mathbb E}}

\def\T{{\mathcal T}}

\let\hat\widehat
\newcommand{\btheta}{\boldsymbol{\theta}}
\newcommand{\bnu}{\boldsymbol{\nu}}

\definecolor{vermilion}{rgb}{0.89, 0.26, 0.2}
\newcommand{\add}[1]{{\color{black} #1}}

\def\train{{\text{train}}}
\def\target{{\text{target}}}

\usepackage[capitalize,noabbrev]{cleveref}

\theoremstyle{plain}

\theoremstyle{definition}

\theoremstyle{remark}

\usepackage[textsize=tiny]{todonotes}

\icmltitlerunning{Classification under Nuisance Parameters and Generalized Label Shift}

\begin{document}

\twocolumn[
\icmltitle{Classification under Nuisance Parameters and Generalized Label Shift \\
in Likelihood-Free Inference}



\icmlsetsymbol{equal}{*}

\begin{icmlauthorlist}
\icmlauthor{Luca Masserano}{equal,cmustats,cmuml}
\icmlauthor{Alex Shen}{equal,cmustats}
\icmlauthor{Michele Doro}{unipd}
\icmlauthor{Tommaso Dorigo}{infn,lulea,usern}
\icmlauthor{Rafael Izbicki}{ufscar}
\icmlauthor{Ann B. Lee}{cmustats,cmuml}
\end{icmlauthorlist}

\icmlaffiliation{cmustats}{Department of Statistics and Data Science, Carnegie Mellon University, Pittsburgh, USA}
\icmlaffiliation{cmuml}{Machine Learning Department, Carnegie Mellon University, Pittsburgh, USA}
\icmlaffiliation{ufscar}{Department of Statistics, Universidade Federal de São Carlos, São Paulo, Brazil}
\icmlaffiliation{infn}{Istituto Nazionale di Fisica Nucleare, Sezione di Padova, Italy}
\icmlaffiliation{lulea}{Lulea Techniska Universitet, Lulea, Sweden}
\icmlaffiliation{usern}{Universal Scientific Education and Research Network, Italy}
\icmlaffiliation{unipd}{Department of Physics and Astronomy, Università di Padova, Padova, Italy}

\icmlcorrespondingauthor{Luca Masserano}{lmassera@andrew.cmu.edu}
\icmlcorrespondingauthor{Ann B. Lee}{annlee@andrew.cmu.edu}

\icmlkeywords{Classification, Label Shift, Latent Shift, Set-Valued Classifier, Simulation-Based Inference, SBI, Likelihood-Free Inference, Nuisance Parameters, Systematic Uncertainties, Hypothesis Testing, Biology, Physics, Science}

\vskip 0.3in
]



\printAffiliationsAndNotice{\icmlEqualContribution} 

\begin{abstract}
An open scientific challenge is how to classify events with reliable measures of uncertainty, when we have a mechanistic model of the data-generating process but the distribution over both labels and latent nuisance parameters is different between train and target data. We refer to this type of distributional shift as {\em generalized label shift} (GLS). 
Direct classification using observed data $\X$ as covariates leads to biased predictions and invalid uncertainty estimates of labels $Y$. We overcome these biases by proposing a new method for robust uncertainty quantification that casts classification as a hypothesis testing problem under nuisance parameters. The key idea is to estimate the classifier's receiver operating characteristic (ROC) across the entire nuisance parameter space, which allows us to devise cutoffs that are invariant under GLS. Our method effectively endows a pre-trained classifier with domain adaptation capabilities and returns valid prediction sets while maintaining high power. We demonstrate its performance on two challenging scientific problems in biology and astroparticle physics with data from realistic mechanistic models.
\end{abstract}
\section{Introduction}\label{sec:introduction}

\paragraph{Problem Set-up} 
Likelihood-free inference refers to settings where the likelihood function $\L(\x;\btheta)$ --- associated with a “theory” or model of the data-generating process --- is intractable, but one is able to simulate relatively large data sets $\T=\{(\btheta_1, \X_1), \ldots, (\btheta_B, \X_B)\} \sim p_\train(\btheta) \L(\x;\btheta)$. These mechanistic models (or simulators) implicitly define \add{the ``causal'' model $\btheta \rightarrow \X$} that encodes our knowledge of how internal parameters determine observable data, and are widely used in several domains of science.

While the likelihood $\L(\x;\btheta)$ stays the same under the assumed theory, the prior over parameters $p_\text{\train}(\btheta)$ is {\em chosen by design} and can be different from the true target distribution $p_{\target}(\btheta)$, thereby causing a potentially harmful bias  when inferring $\btheta$ given a new observation $\x_{\target}$. If the unknown parameter of interest is a categorical variable $Y \in \mathcal{Y}=\{0, 1, \ldots, K\}$ and the causal mechanistic model remains the same --- that is, $p_{\train}(\X \mid Y)=p_{\target}(\X \mid Y)$ --- the difference in the joint distribution of $(\btheta, \X)$ between train and target data is referred to as prior probability shift or label shift \citep{quinonero2008dataset, vaz2019quantification, polo2023unified, storkey2009training, fawcett2005response, moreno2012unifying}. We refer to this setting as {\em standard label shift} (SLS). 

In this paper, we consider a more general setup that reflects a richer mechanistic model: \add{$\btheta = (Y, \bnu) \rightarrow \X$}, where $\bnu \in \Ncal$ are \add{continuous or discrete} nuisance parameters that are not of direct interest but critically influence the data-generating process. These nuisance parameters are available at the training stage, but are {\em not} observed at the inference stage when estimating $Y$ from $\x_{\target}$. We refer to a shift that simultaneously affects $Y$ and $\bnu$ as {\em generalized label shift} (GLS), and assume that $p_{\train}(\X \mid Y,\bnu)=p_{\target}(\X \mid Y,\bnu)$. \add{Within this setting, our goal is not just to do binary classification per se (that is, providing a $0$ versus $1$ response), but rather to do trustworthy uncertainty quantification for the classification output, even under GLS.}
 
\paragraph{Scientific Motivation}
Nuisance parameters can be seen as a way of accounting for model misspecifications. Statistical models are indeed rarely accurate in capturing the complexity of physical phenomena. To account for “known unknowns”, such as calibration errors in the measuring device or inaccuracies and approximations in the theory, scientists usually resort to enlarging the mechanistic model with additional parameters that are not of direct relevance, but yet have to be considered during inference in order to make reliable statements about the parameters of interest. These additional parameters are commonly referred to as nuisance parameters \cite{kitching2009cosmological,dorigo2020dealing, pouget2013probabilistic, hepmllivingreview}: \add{they are necessary to achieve more faithful models of reality, but make correct inference much more challenging.}

\begin{figure}[b!]
    \centering
    \includegraphics[width=\columnwidth]{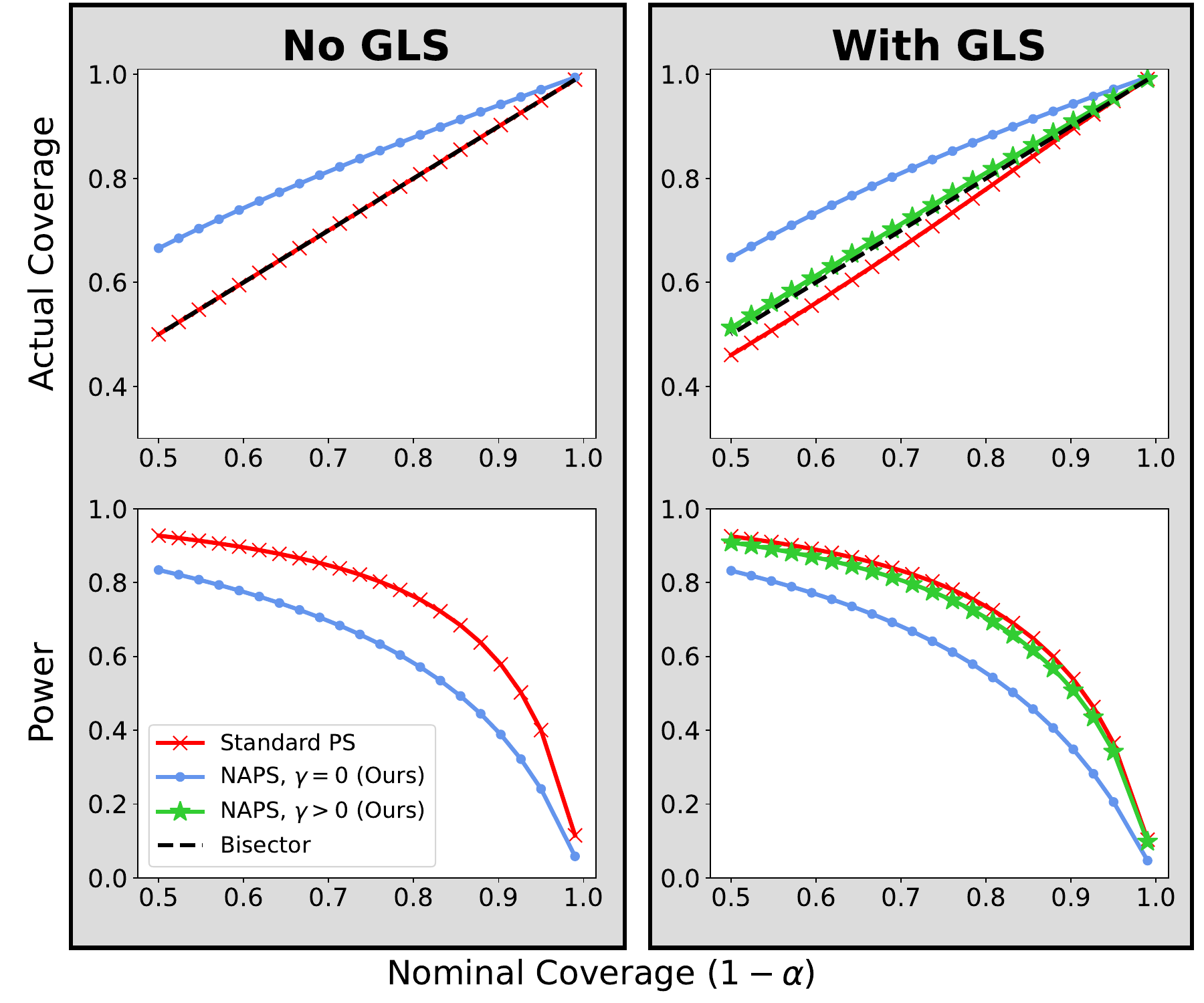}
    \caption{\textbf{Synthetic Example.} {\em Left (no GLS):} Standard prediction sets $R_\alpha(\x)$ (\textcolor{red}{red}) guarantee marginal coverage at the nominal level. Nuisance-aware prediction sets (NAPS $\gamma = 0$; \textcolor{NavyBlue}{blue}) are also marginally valid, but the ``universality'' of  conditional validity across the entire nuisance parameter space comes at the price of more conservative prediction sets and lower power. {\em Right (with GLS):} Standard prediction sets are no longer valid and undercover for all $\alpha$ levels (red curve is below the black bisector), while NAPS are still valid. Furthermore, we can increase power while maintaining validity (NAPS $\gamma > 0$; \textcolor{LimeGreen}{green}) by constructing $(1-\gamma)$ confidence sets of the nuisance parameter $\nu$ and deriving less conservative cutoffs given an observation. Here $\gamma = \alpha \times 0.01$.}
    \label{fig:synth2}
\end{figure}

\paragraph{Statistical Challenges} We introduce a simplified example (see Section~\ref{sec:synthetic_example} for details) to illustrate the challenges of classification under the presence of nuisance parameters. Suppose $Y=1$ represents a class with cases of interest (e.g., the presence of a medical condition) and $Y=0$ a class with cases of no interest. We have good knowledge of the probability density function (PDF) of $Y=1$, $f_1(x)$, but the shape of the distribution of $Y=0$ is largely unknown. To accommodate different scenarios, we resort to a nuisance-parameterized PDF $f_0(x; \nu)$. Our goal is to discriminate between negative $Y=0$ and positive $Y=1$ cases based on potentially high-dimensional data $\x \in \mathcal{X}$ and to provide valid measures of uncertainties on the true label $Y$ under the presence of a nuisance parameter $\nu$. However, directly classifying $\x_{\target}$ based on $\P_{\train}(Y=1 \mid \X)$ and a cutoff $C$ derived from $\mathcal{T}= \{(Y_i, \X_i)\}_{i=1}^{B}$ would lead to invalid uncertainty quantification. Indeed, under GLS (or even SLS), standard prediction sets (defined as in, e.g., Equation~\ref{eq:standardPS}) do not guarantee marginal validity:
\begin{equation*}
    \P_{\target}(Y \in R_\alpha(\X)) \geq 1-\alpha,
\end{equation*}
where $Y$ and $\X$ are random and $\alpha \in [0,1]$ is a pre-specified miscoverage level. Various solutions have been proposed for the SLS setting (see references in Section~\ref{sec:related_work}), whereas GLS is still a largely unexplored area in the machine learning literature. The key open challenge is to design general-purpose inference algorithms that can guarantee {\em valid} measures of uncertainty for all $Y$ and $\bnu$ while providing high constraining power on $Y$ (that is, smaller prediction sets). 

Returning to our simplified experiment, Figure~\ref{fig:synth2} (top left) illustrates how standard prediction sets $R_\alpha(\x)$ are marginally valid when the train and target distributions are the same, while under GLS prediction sets are no longer valid even marginally (top right). Our nuisance-aware prediction sets (NAPS, $\gamma = 0$ in Figure~\ref{fig:synth2}), on the other hand, are valid in both settings. In addition, we can increase the constraining power (NAPS, $\gamma>0$) once we observe data without the need to re-train the classifier, effectively endowing our method with domain adaptation capabilities. 

\paragraph{Approach and Contributions} We categorize our main contributions as follows: 

\textit{i) {\bf TPR and FPR across $\Ncal$}. } By casting classification under GLS as a hypothesis testing problem with nuisance parameters, we propose a method to estimate the TPR and FPR curves across the nuisance parameter space via monotone regression. This allows us to compute the entire receiver-operating-characteristic (ROC) of the classifier for all $\bnu \in \Ncal$ (Section \ref{sec:RejectionAndROC} and Algorithm \ref{alg:power_function}). 

\textit{ii) }{\bf Nuisance-aware prediction sets (NAPS). } Rather than providing a point prediction based on an estimate of $\P_{\train}(Y=1 \mid \X)$, we derive selection criteria that are valid under GLS and construct a {\em set-valued classifier} $\H: \x \mapsto \{\emptyset, 0, 1, \{0,1\}\}$ which guarantees that the true label is included in the set with probability at least $(1-\alpha)$, regardless of the true class $y$ and of the value of the nuisance parameters $\bnu$. That is, the prediction sets $\mathbf{H}_\alpha(\X)$ guarantee conditional validity under GLS (Theorem \ref{thm:nacs_coverage}):
\begin{equation}\label{eq:cond_coverage}
    \P_\target(Y \in \H_\alpha(\X) \mid y, \bnu) \geq 1-\alpha,  \ \forall y \in \mathcal{Y}, \ \bnu \in \mathcal{N}.
\end{equation}
Standard point classifiers (e.g., the Bayes classifier; Appendix~\ref{app:bayes_clf}) and prediction sets based on $\P_\train(Y=1 \mid \X)$ are not conditionally valid across the nuisance parameter space, and hence are also not valid marginally under GLS. On the other hand, our algorithm returns valid NAPS for all levels $\alpha \in (0,1)$ simultaneously given any new observation $\x_{\target}$ without having to retrain the classifier. This also yields marginal validity under GLS (Theorem \ref{thm:nacs_coverage}). Our results do \textit{not} rely on asymptotic theory with the number of observations $n\rightarrow \infty$.  We only assume to have a sufficient number of simulations $B$ to train and calibrate the classifier.

\textit{iii) }{\bf NAPS with higher power. } We show how one can further increase power while maintaining validity by constraining nuisance parameters given an observed $\x_{\target}$ through $(1-\gamma)$ confidence sets of the nuisance parameters $\bnu$, where $\gamma$ is a small pre-defined error level. This effectively allows to derive data-dependent cutoffs that decrease the average size of prediction sets given a specific observation.

We demonstrate our method using data from two high-fidelity scientific simulators: \texttt{scDesign3} \citep{song2023scdesign3} which generates realistic single-cell RNA-sequencing data, and \texttt{CORSIKA} \citep{heck1998corsika} which models the interactions of primary cosmic rays with the Earth's atmosphere. A flexible implementation of NAPS is available at \href{https://github.com/lee-group-cmu/lf2i}{https://github.com/lee-group-cmu/lf2i}.
\section{Related Work} \label{sec:related_work}

To the best of our knowledge, this is the first work that estimates ROC curves across the entire parameter space $\Theta=\mathcal{Y} \times \mathcal{N}$. To construct {\em frequentist confidence sets}, we base our results directly on the class probability $\P_\train(Y=1 \mid \X)$, rather than using a surrogate likelihood or likelihood ratio (see for example references in \citealt{cranmer2020frontier}). The idea of \textit{improving power} of  NAPS with $\gamma>0$ is similar to \citet{berger1994p}, and close in spirit to likelihood profiling, with the key difference that profiling does not guarantee validity (even for a large number of simulations $B$ and under no GLS), and also requires an approximation of the likelihood and the maximum likelihood estimate of $\bnu$. The ROC {\em calibration} framework of Section \ref{sec:RejectionAndROC} is related to \citet{zhao2021diagnostics} and \citet{dey2022calibrated}, which use monotone regression to estimate the CDF of probability integral transforms for calibrating posterior probabilities, but not for constructing valid prediction sets under GLS. When the prior distribution over $y$ in the target data is known,  $\P_\train(Y=1 \mid \X)$ can be easily recalibrated to match  $\P_\target(Y=1 \mid \X)$ under SLS \citep{saerens2002adjusting, lipton2018detecting}. However, this is not possible under GLS  since $\bnu$ is unknown at inference time. Moreover, our approach does not assume such a known prior. The construction of {\em set-valued classifiers} of Section~\ref{sec:set_classifiers} is inspired by \citet{sadinle2019least, dalmasso2021likelihood, masserano2023simulator}. There are also connections to {\em conformal prediction}: Conformal methods are widely used because they ensure prediction sets with marginal coverage when data are exchangeable \citep{papadopoulos2002inductive, Vovk2005, Lei2018}. However, conformal methods need adjustments under distributional shift when data are no longer exchangeable. Such adjustments need to be tailored for the type of shift at hand \citep{tibshirani2019conformal}. For instance, label shift can be addressed through label-conditional conformal prediction \citep{vovk2014conformal, vovk2016criteria, sadinle2019least},
which guarantees coverage conditional on the label $y$ \citep[Section~2.2]{podkopaev2021distribution} under SLS, but not under the presence of nuisance parameters and GLS. Finally, our work directly addresses the existing gap in methods for constructing {\em reliable simulator-based inference} algorithms with valid uncertainty quantification guarantees \citep{hermans2021averting}. Our work is also inspired by the vast literature in high-energy physics on hypothesis testing and {\em nuisance-parameterized} machine-learning methods \citep{Feldman1998UnifyingApproach, cousins2006treatment, Sen2009NuisanceParameters, chuang1998hybridresampling, louppe2017learning, cowan2011asymptotic}, which also includes the so-called “mining gold” idea of leveraging hidden information on latent variables in an all-knowing simulator \citep{brehmer2020mining}.

\section{Methodology}\label{sec:method}

For simplicity, we will restrict our discussion to $Y \in \{0,1\}$.

\subsection{Classification as  Hypothesis Testing}\label{sec:classification_hp}

We reformulate the binary classification problem as a composite-versus-composite hypothesis test: 
\begin{equation} \label{eq:gen_test}
    H_{0,y}: \btheta \in \Theta_0   \ \ \text{versus} \ \ H_{1, y}: \btheta \in \Theta_1,
\end{equation}
where $\Theta_0= \{y\} \times \mathcal{N}$, $\Theta_1= \{y\}^c \times \mathcal{N}$.
We define
\begin{equation} \label{eq:BF_test_statistic}
\tau_{y}(\x) =
\frac{\P_\train(Y=y \mid \x) \; \P_\train(Y \neq y)}{\P_\train(Y \neq y \mid \x) \; \P_\train(Y=y)}
\end{equation}
as our test statistic, which is equivalent to the Bayes factor for the test in Equation~\ref{eq:gen_test}; see Appendix \ref{sec:app_bayes_factor} for a derivation.  
Alternatively, one can define the test statistic as the probabilistic classifier $\P_\train(Y=y|\x)$ itself. Both quantities (which are related via a monotonic transformation) can be estimated directly from a {\em pre-trained} classifier based on $\mathcal{T}_B$.  That is, there is no need for an extra step to, e.g., learn the likelihood function $\mathcal{L}(\x;Y, \bnu)$ or the associated likelihood ratio statistic from simulated data as done in   \citet{cranmer2020frontier}, \citet{rizvi2023learning}, and references therein.

We denote  the estimate of $\tau_{y}$ by $\widehat{\tau}_{y}$ and reject the null $H_{0,y}$ for small values of $\widehat{\tau}_{y}$. For example, if the null represents $y=0$, then a ``positive'' case ($y=1$) in binary classification would correspond to small values of $\widehat{\tau}_{0}$, or equivalently, large values of the probabilistic classifier $\widehat{\P}_\train(Y=1 \mid \x)=1-\widehat{\P}_\train(Y=0 \mid \x)$.
In this work, we define cutoffs for $\widehat{\tau}_{y}$ so that prediction sets are approximately valid under nuisance parameters and GLS.

\subsection{The Rejection Probability Across the Entire Parameter Space}\label{sec:RejectionAndROC}

To choose the optimal cutoff to reject $H_{0,y}$ and construct valid prediction sets, we need to know how the classifier performs for different values of the nuisance parameters $\bnu$. The first step is to compute the following quantity:

\begin{Definition}[Rejection probability] 
\label{def:reject_prob} 
Let $\lambda$ be any test statistic, e.g, the estimated Bayes factor, $\lambda = \widehat{\tau}_{y}$. The rejection probability of $\lambda$ is defined as
\begin{align}
\label{eq:reject_prob}
W_{\lambda}(C; y, \bnu) := \pr_\target \left(  \lambda(\X) \leq C |y,\bnu\right),
\end{align}
where $y \in \{0,1\}$, $\bnu \in \mathcal{N}$, and   $C \in \mathbb{R}$.
\end{Definition}

For fixed $\bnu$ and null $H_{0,0}:Y=0$, the receiver operating characteristic (ROC) relates the true positive rate 
$$\texttt{TPR}(C; \bnu):=W_{\widehat  \tau_{0}}(C; 1, \bnu)$$
to the false positive rate
$$\texttt{FPR}(C; \bnu):= W_{\widehat \tau_{0}}(C; 0, \bnu),$$
while varying the cutoff $C$.
Figure~\ref{fig:cosmic_ROC} shows examples of some ROC curves at different values of $\bnu$ when the null represents the negative class $y=0$, for the setting of Section~\ref{sec:cosmic_rays}.

A key insight behind our method is that the rejection probability (Equation~\ref{eq:reject_prob}) is invariant under GLS even if estimated from $p_{\train}$; in other words, it is always the same for train and target data (Lemma \ref{lemma:reject_prob_invariance}). As a result, our ROC curves reliably measure the performance of the classifier under nuisance parameters. In practice, we can estimate $W_{\lambda}(C; y, \bnu)$ for all $y$ and $\bnu$ simultaneously using regression with a monotonic constraint in $C$. The whole procedure is amortized with respect to the target data, meaning that both the base classifier and the rejection probability are estimated only once, after which they can be evaluated on an arbitrary number of observations.

\subsection{Selecting the Optimal Cutoff under GLS}\label{sec:ROC_cutoffs}

Once we know the classifier's rejection probability function, we can apply it in various ways. All our choices are robust against GLS.

\paragraph{Controlling FPR or TPR } Based on $W_{\lambda}(C; y, \bnu)$, we can find the cutoff $C$ for a new test point that either controls type-I error (FPR), or guarantees a minimum recall (TPR), or maximizes some other metric of choice that depends on both FPR and TPR. For example, FPR control at some pre-specified level $\alpha \in [0,1]$ and $\bnu_0 \in \mathcal{N}$ implies $C_{\alpha} = \texttt{FPR}^{-1}(\alpha; \bnu_0),$  and TPR control at some minimum recall $\alpha$ implies $\widetilde{C}_{\alpha} = \texttt{TPR}^{-1}(\alpha; \bnu_0)$. To control FPR or TPR {\em uniformly} over $\bnu$, one can instead choose $C_{\alpha} = \inf_{\bnu \in \mathcal{N}} \texttt{FPR}^{-1}(\alpha; \bnu),$ and $\widetilde{C}_{\alpha} = \sup_{\bnu \in \mathcal{N}} \texttt{TPR}^{-1}(\alpha; \bnu)$, respectively. Although robust under GLS, such cutoffs can be overly conservative.

\paragraph{Controlling FPR or TPR, but with more power} An alternative approach, which is still valid for any $\bnu$ and can increase power, is to restrict the search over nuisance parameters to a smaller region of $\mathcal{N}$. For this approach, we first construct a confidence set $S(\x;\gamma)$ for $\bnu$ and fixed $y \in \{0,1\}$ at a pre-specified $(1-\gamma)$ level (Definition~\ref{def:CS_nuisance}). This allows to choose a data-dependent cutoff such that 
 \begin{equation*}
           C_\alpha^*(\x) = \inf_{\bnu \in S(\x;\gamma)} \{ \texttt{FPR}^{-1}(\beta;\bnu)\},
\end{equation*}
where $\beta=\alpha-\gamma$, where the minimization is over the restricted set $S(\x;\gamma) \subseteq \mathcal{N}$. In practice, $S(\x;\gamma)$ can be either obtained from \add{auxiliary} measurements that are available at inference time, or from a separate pre-trained model that returns valid confidence sets on $\bnu$ from data $\x$. Lemma~\ref{lemma:NA_cutoff} demonstrates that this cutoff guarantees a maximum type-I error equal to $\alpha$ (FPR control) for any $\bnu \in \mathcal{N}$. Similarly, for TPR control, choosing $\widetilde{C}_{\alpha}^*(\x) = \sup_{\bnu \in S(\x;\gamma)} \texttt{TPR}^{-1}(\beta; \bnu)$ with $\beta=\alpha+\gamma$ guarantees a minimum recall of at least $\alpha$. The special case of $\gamma=0$ (and $\beta=\alpha$) corresponds to $S(\x;\gamma) = \mathcal{N}$; that is, no constraints on the nuisance parameters. Finally, note that hybrid cut-offs $\texttt{FPR}^{-1}(\beta; \hat{\bnu})$ and $\texttt{TPR}^{-1}(\beta; \hat{\bnu})$ based on a {\em point prediction} $\hat{\bnu}(\x)$ of the nuisance parameters (such as the posterior mean) would not lead to valid uncertainty quantification under GLS (see Figure~\ref{fig:synth_covextra} in Appendix). 

\subsection{Constructing Robust Set-Valued Classifiers}\label{sec:set_classifiers}

\begin{algorithm}[t!]
    \caption{\texttt{Nuisance-aware prediction sets}}\label{alg:naps}
    
    \textbf{Input: }{\small training set $\mathcal{T}= \{(Y_i, \Xbf_i)\}_{i=1}^{B}$; calibration set \add{$\mathcal{T}^\prime=\{(Y^\prime_i, \bnu^\prime_i, \X^\prime_i)\}_{i=1}^{B^\prime}$}; observation $\xbf$; test statistic $\lambda = \tau_y$; miscoverage levels $\alpha \in [0, 1]$ and $\gamma \in [0, \alpha]$.\\}
    \textbf{Output: }{\small Prediction set $H_\alpha(\xbf)$ such that Equation~\ref{eq:cond_coverage} holds.\\}		
    \begin{algorithmic}[1]
        \STATE \codecomment{Training}
        \STATE Estimate $\P_\train(Y=y \mid \Xbf)$ via a probabilistic classifier
        \STATE \codecomment{Calibration}
        \STATE Estimate $W_{\tau_{y}}(C; y, \bnu) := \pr_\target \left(  \tau_{y}(\X) \leq C \mid y,\bnu\right)$ as detailed in Algorithm~\ref{alg:power_function} by
        \begin{enumerate}[i.]
            \item \add{Computing the test statistic $\hat{\tau}_{y}(\x)$ as in Equation~\ref{eq:BF_test_statistic} for all $\X \in \mathcal{T}^\prime$;}
            \item Constructing the augmented calibration set $\mathcal{T}^{\prime\prime}$;
            \item Estimating \add{the rejection probability function} $W_{\hat{\tau}_{y}}(C; y, \bnu)$ from $\mathcal{T}^{\prime\prime}$ via monotone regression.
        \end{enumerate}
        \STATE \codecomment{Inference}
        \FOR{$y \in \{0, 1\}$}
            \STATE Compute $\hat{\tau}_{y}(\xbf)$ as in Equation~\ref{eq:BF_test_statistic}
            \IF{$\gamma = 0$}
                \STATE $C_{\alpha, y}^*(\x) \gets \inf_{\bnu \in \mathcal{N}} \{ \hat{W}_{\hat{\tau}_{y}}^{-1}(\alpha;y,\bnu)\}$
            \ELSE 
                \STATE \add{Constrain nuisance parameters by constructing a} level-$\gamma$ confidence set $S_y(\x;\gamma)$ for $\bnu$
                \STATE $C_{\alpha, y}^*(\x) \gets \inf_{\bnu \in S_y(\x;\gamma)} \{ \hat{W}_{\hat{\tau}_{y}}^{-1}(\alpha-\gamma;y,\bnu)\}$
            \ENDIF
        \ENDFOR
        \STATE $\H(\x;\alpha) \gets \left\{ y \in \{0,1\} \mid \widehat{\tau}_{y}(\x) > C_{\alpha,y}^*(\x) \right\}$
        \STATE \textbf{return} Prediction set $\H(\x;\alpha)$ for $Y$
    \end{algorithmic}
\end{algorithm}

Rather than just returning a single label $0/1$ for each observation $\x$ like the standard Bayes classifier (Appendix~\ref{app:bayes_clf}), our method yields prediction sets from a set-valued classifier. 

\begin{Definition}[Nuisance-aware prediction set]\label{def:NAPS}
A nuisance-aware prediction set (NAPS) is the set returned from a  set-valued classifier $\H: \x \mapsto \{\emptyset, 0, 1, \{0,1\}\}$ with 
\begin{equation}
\label{eq:NAC}
        \H(\x;\alpha) = \left\{ y \in \{0,1\} \mid \widehat{\tau}_{y}(\x) > C_{\alpha,y}^*(\x) \right\},
\end{equation} 
where
 \begin{equation}
    C_{\alpha,y}^*(\x) = \inf_{\bnu \in S_{y}(\x;\gamma)} \{ W^{-1}_{\widehat \tau_{y}}(\beta; y, \bnu)\},
\end{equation}
is the rejection cutoff, $\beta=\alpha-\gamma$ 
and $S_y(\x;\gamma)$ is a $(1-\gamma)$ confidence set for $\bnu$ defined by Equation~\ref{eq:CS_nuisance}.
\end{Definition}

This classifier guarantees user-defined levels of coverage $1-\alpha$ (the probability that the true label is included in the set), no matter what the true class $y$ and the nuisance parameters $\bnu$ are (Theorem~\ref{thm:nacs_coverage}). \add{The resulting prediction sets contain all labels that were not rejected by the corresponding hypothesis test. Ambiguous sets can arise in two cases: \textit{i)} When both null hypotheses are rejected, we obtain an empty set. However, empty sets only arise at very low confidence levels (high values of $\alpha$), which is typically not considered an interesting regime; \textit{ii)} When both null hypotheses are accepted, we obtain a prediction set that includes both 0 and 1. This latter type of ambiguity reflects the uncertainty of the classifier, which typically grows at higher confidence levels (low values of $\alpha$). A low-quality classifier will often report an “I-don’t-know answer” for ambiguous instances if forced to guarantee a certain confidence level, rather than returning a $0/1$ answer that has a high chance of being incorrect.}

While $\gamma=0$ can be the default choice \add{for NAPS}, \add{choosing a small} $\gamma>0$ \add{often leads} to higher power (see Section~\ref{sec:experiments}). Finally, note that while our set-valued classifier targets conditional coverage under GLS according to Equation~\ref{eq:cond_coverage}, as a by-product we also achieve prediction sets with marginal coverage under GLS (see Theorem \ref{thm:nacs_coverage}).

\add{Algorithm~\ref{alg:naps} includes a step-by-step description of the entire procedure for constructing nuisance-aware prediction sets.}
\section{Theoretical Results}\label{sec:theory}

Proofs for this section can be found in Appendix \ref{sec:poofs}.

\subsection{Validity and Robustness to GLS}

\begin{Lemma}[Invariance of the Rejection Probability to GLS]
\label{lemma:reject_prob_invariance}
Under GLS, the rejection probability (Definition \ref{def:reject_prob}) of any test statistic $\lambda$ 
is invariant to GLS, that is
\begin{equation*}
    \begin{split}
        W_{\lambda}(C; y, \bnu) &= \pr_\target  \left(  \lambda(\X) \leq C \mid y,\bnu \right)\\ 
        &= \pr_\train \left(  \lambda(\X) \leq C \mid y,\bnu\right).
    \end{split}
\end{equation*}
\end{Lemma}

\subsubsection{Nuisance-Aware Cutoffs}

\begin{Definition}[Confidence set for nuisance parameters] \label{def:CS_nuisance}

The random set $S_y(\x;\gamma)$ is a valid $(1-\gamma)$ level confidence set for $\bnu$ at fixed $y \in \{0,1\}$, if 
\begin{equation}
		 \P_\target \left(\bnu \in S_{y}(\X;\gamma) \mid y,\bnu \right) \geq 1-\gamma , \ \ \forall \bnu \in \mathcal{N} , \label{eq:CS_nuisance}
	 \end{equation}
  for some pre-specified value $\gamma \in [0,1]$.
\end{Definition}

The following theorem shows that nuisance-aware cutoffs control FPR and TPR at the specified level.

\begin{thm}[Nuisance-aware cutoffs for FPR/TPR control]\label{lemma:NA_cutoff}
Choose a threshold $\alpha \in [0, 1]$ and $\gamma \in [0, \alpha]$. Let $S_y(\x;\gamma)$ be a valid $(1-\gamma)$ confidence set for $\bnu$ at fixed $y \in \{0,1\}$ according to Definition~\ref{def:CS_nuisance}. Let $\lambda(\X)$ be any test statistic that measures how plausible it is that $\X$ was generated from $H_{0,y}$. Define the nuisance-aware rejection cutoff to be
\begin{equation} \label{eq:cutoff_FPR}
           C_{\alpha,y}^*(\x) = \inf_{\bnu \in S_{y}(\x;\gamma)} \{ W^{-1}_\lambda(\beta; y, \bnu)\},
\end{equation}
where $\beta=\alpha-\gamma$, and $W$ is the rejection probability in Definition~\ref{def:reject_prob}. 
Then, for all $\bnu \in \mathcal{N}$, we have FPR control:
 \begin{align}\label{eq:FPR_control}
 \P_\target &\left(\lambda{(\X}) \leq C_{\alpha,y}^*(\X) \mid y,\bnu \right) \leq \alpha \\
 &\ \ \ \text{(maximum type-I error probability for $H_{0,y}$)}. \notag
 \end{align}
Similarly, if 
$$\widetilde{C}_{\alpha,y}^*(\x) = \sup_{\bnu \in S_{1-y}(\x;\gamma)} \{ W^{-1}_\lambda(\beta; 1-y, \bnu)\},$$
with  $\beta=\alpha+\gamma$,
then for all $\bnu \in \mathcal{N}$, we have TPR control:
 \begin{align*}
\P_\target &\left(\lambda(\X) \leq \widetilde C^*_{\alpha,y}(\X) \mid 1-y,\bnu\right) \geq \alpha \\ 
&\ \  \ \   \ \text{(minimum recall for $H_{0,y}$)}.
  \end{align*}
\end{thm}

\subsubsection{Properties of the Nuisance-Aware Prediction Set}

The nuisance-aware prediction set (Definition~\ref{def:NAPS}) is both {\em conditionally} and \textit{marginally} valid with respect to both $y$ and $\bnu$ under GLS. 

\begin{thm} 
\label{thm:nacs_coverage}
Let $\H(\x;\alpha)$ be the  nuisance-aware  prediction set  of Definition \ref{def:NAPS}. Under GLS, for every $y \in \{0,1\}$ and $\nu \in \mathcal{N}$ 
$$\P_\target(Y \in \H(\X;\alpha) \mid y,\bnu)\geq  1-\alpha.$$
Moreover,
$$\P_\target(Y \in \H(\X;\alpha))\geq  1-\alpha.$$
\end{thm}

\section{Experiments}\label{sec:experiments}

\subsection{Synthetic Example}\label{sec:synthetic_example}

Consider a simplified setting where we are certain about the data-generating process of $Y=1$ cases of interest, but not about that of $Y=0$ cases. We assume
\begin{align*}
    p(x_i \mid Y_i = 1) &= \frac{e^{x_i}}{e-1} \\
    p(x_i \mid Y_i = 0, \nu_i) &= \frac{\nu_i e^{-\nu_i x_i}}{1-e^{-\nu_i}},
\end{align*}
where $\nu \in [1, 10]$ is a nuisance parameter, which enlarges the model for $Y=0$ to reflect our uncertainty of how cases of no direct interest might manifest themselves.

\paragraph{Before Data Collection} Before having specific knowledge about target data and experimental conditions, we decide to draw $\nu$ from a uniform reference distribution $p_\train(\nu)=\mathcal{U}[1,10]$ (here $\P_\train(Y=1)=\P_\target(Y=1)=0.5$ is fixed). We then pre-train a classifier\footnote{In this simplified example we can actually compute everything semi-analytically.} and compute the class posterior $\P_\train(Y=1|x)$, and construct $(1-\alpha)$ prediction sets 
\begin{align}
    \label{eq:standardPS}
R_\alpha(x) := \{y : \P_\train(Y = y \mid x) > C_{\alpha}^* \}\end{align}
with cutoffs 
$$C_{\alpha}^* \ \ \text{s.t.} \ \  \pr_\train \left(  \P_\train(Y = y \mid X) \leq C_{\alpha}^* \right) = \alpha,$$
for a pre-specified miscoverage level $\alpha$.
These are the oracle prediction sets that minimize ambiguity (i.e., average size) subject to having the correct total coverage according the Theorem 1 from \citet{sadinle2019least}. We will henceforth refer to them as ``standard prediction sets'' to distinguish them from the oracle class-conditional prediction sets from \citet{sadinle2019least} and NAPS.

\paragraph{Setting 1: No GLS} When train and target data have the same distributions, the prediction sets $R_\alpha(X)$ have guaranteed marginal coverage 
$$\P_\train(Y \in R_\alpha(X)) = 1-\alpha$$
at the nominal $(1-\alpha)$ level by construction (red curve overlapping black bisector in Figure \ref{fig:synth2}, top left), although they might still undercover in specific regions of the nuisance parameter space (see Figure \ref{fig:synth_covextra} in Appendix \ref{sec:deep-dive}). NAPS with $\gamma=0$ are instead both marginally valid (blue curve, top left) and conditionally valid (Theorem~\ref{thm:nacs_coverage}). The latter “universality” can cause overly conservative prediction sets and a loss of power (defined as the probability of rejecting $H_{0,y}: Y = y$ when $Y \neq y$); see bottom left panel.

\paragraph{Setting 2: With GLS} Suppose now that we apply the pre-trained classifier to a target distribution with a {\em different} distribution over the nuisance parameters, namely $p_\target(\nu)=N(4, 0.1) \neq p_\train(\nu)$. The top right panel of Figure~\ref{fig:synth2} shows that the prediction sets $R_\alpha(X)$ are no longer valid even marginally (red curve below bisector), whereas NAPS are still valid. Moreover, we can achieve higher power by constraining the optimization to a high-confidence set of the nuisance parameter (compare green with blue NAPS curves). In summary: our proposed method can leverage the original  $\P_\train(Y= 1 \mid \x)$ classifier to provide prediction sets that are both valid and precise for any distribution $p(y,\nu)$ as long as $x|y,\nu$ stays the same. Additional results for other prediction set methods and NAPS with $\gamma > 0$ are available in Appendix \ref{sec:deep-dive}.

\subsection{Single-Cell RNA Sequencing}\label{sec:rna_seq}

\begin{figure}[t!]
    \centering
    \includegraphics[width=1\columnwidth]{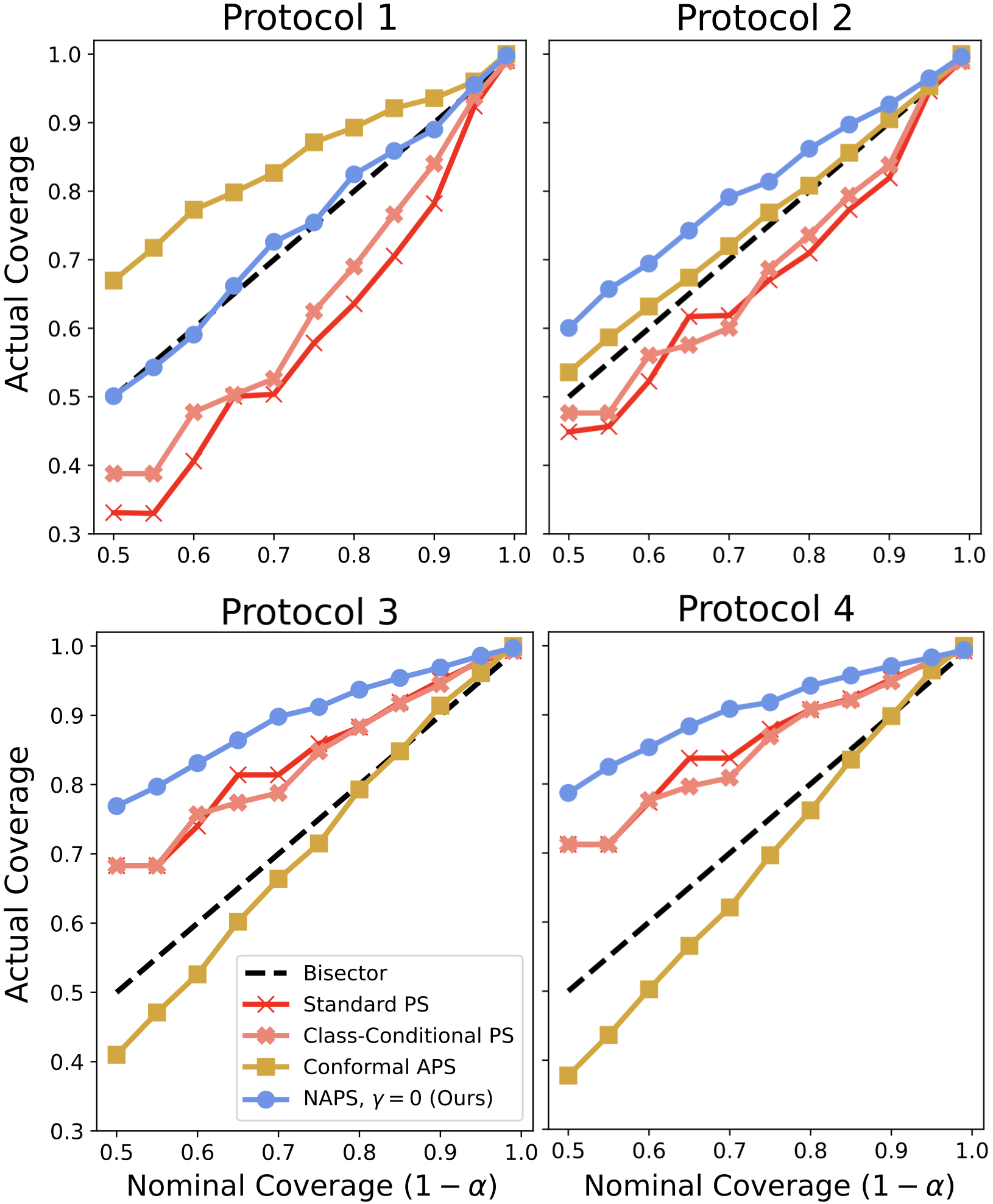}
    \caption{\textbf{Coverage under different batch protocols $\nu$ for the RNA-Seq example.} Each marker represents the proportion of samples in the test set whose true label was included in the constructed prediction sets. Nuisance-aware prediction sets (NAPS $\gamma=0$; \textcolor{NavyBlue}{blue}) are valid regardless of the protocol, which is unknown at inference time. All other methods for prediction sets with marginal coverage (\textcolor{BrickRed}{red}), class-conditional coverage (\textcolor{Melon}{pink}), and conformal adaptive prediction sets (\textcolor{YellowOrange}{gold}) undercover for at least two batch protocols.}
    \label{fig:rna_seq_cond_coverage}
\end{figure}

RNA sequencing, or RNA-Seq, is a vital technique in genetics and genomics research that has revolutionized our understanding of gene expression. Many RNA-seq experiments involve extracting RNA from target cells and examining counts of specific genes. While the natural variation in gene counts between different types of cells is interesting to researchers, the observed gene counts depend also on the precise steps of the sequencing process. For example, the exact chemicals, equipment, room temperature and lab technician can greatly influence the final measurements, in addition to the cell type. In practice, these so-called “batch effects” are often unmeasured confounders whose exact value is unknown at the inference stage. Thus, analysis of experimental gene counts must take them into account in order to conduct reliable scientific analysis. In what follows, we define a “batch protocol” to be a particular set of these conditions common to a batch of cells.

We use data from the recently proposed \texttt{scDesign3} simulator \citep{song2023scdesign3}, with reference data taken from the PBMC Systematic Comparative Analysis \citep{ding2019systematic}. We consider two cell types (\texttt{CD4\textsuperscript{+}} T-cells and \texttt{Cytotoxic} T-cells) and a subset of 100 random genes. The reference data contains counts from two separate experiments, which will serve as the basis of our simulated batch protocols. We use the two original experimental conditions as well as two artificial perturbations derived from them to generate four possible batch protocols. Following our terminology, this corresponds to a discrete nuisance parameter with four groups. We consider the setup of a classifier trained on data from all four possible protocols and tested on different $\x_{\target}$ whose true protocol value is unknown (in addition to the cell type). In total, we have available $80{,}000$ samples which we divide into train ($60\%$), calibration ($35\%$) and test ($5\%$) sets. Our goal is to infer the cell's type from the observed gene count under the presence of the unknown nuisance parameter. 

We compare our method with three baselines: (i) standard prediction sets for which cutoffs are computed from $\mathbb{P}(Y|\X)$ \citep[Theorem 1]{sadinle2019least}; (ii) class-conditional prediction sets with cutoffs derived separately from each $\mathbb{P}(Y=i|\X),$ $i \in \{0,1\}$ \citep{sadinle2019least}; and (iii) conformal adaptive prediction sets (APS; \citet{romano2020classification}). Figure~\ref{fig:rna_seq_cond_coverage} shows that nuisance-aware prediction sets (NAPS) are valid regardless of the protocol, which is unknown at inference time. On the other hand, all of the other prediction sets from the analyzed baselines undercover for at least two protocols. Nuisance-aware cutoffs need to control type-I error for every single value of the nuisance parameter, including the hardest case. Here, Protocol 1 (top left) appears to be the most difficult to classify correctly. Finally, we note that while conformal APS approximately achieves coverage for $(1-\alpha) \approx 1$, this comes at the expense of uninformative prediction sets that contain both labels for all $\x_{\target}$. NAPS, on the other hand, is able to maintain high power (see Figure \ref{fig:rna_wtp} in Appendix~\ref{app:rna_seq}). Additional results and details on the base classifier, the model used to estimate the rejection probability function, and the baselines adopted for comparison can be found in Appendix~\ref{app:rna_seq}.

 \subsection{Atmospheric Cosmic-Ray Showers}\label{sec:cosmic_rays}

\begin{figure}[b!]
    \centering
    \includegraphics[width=1\columnwidth]{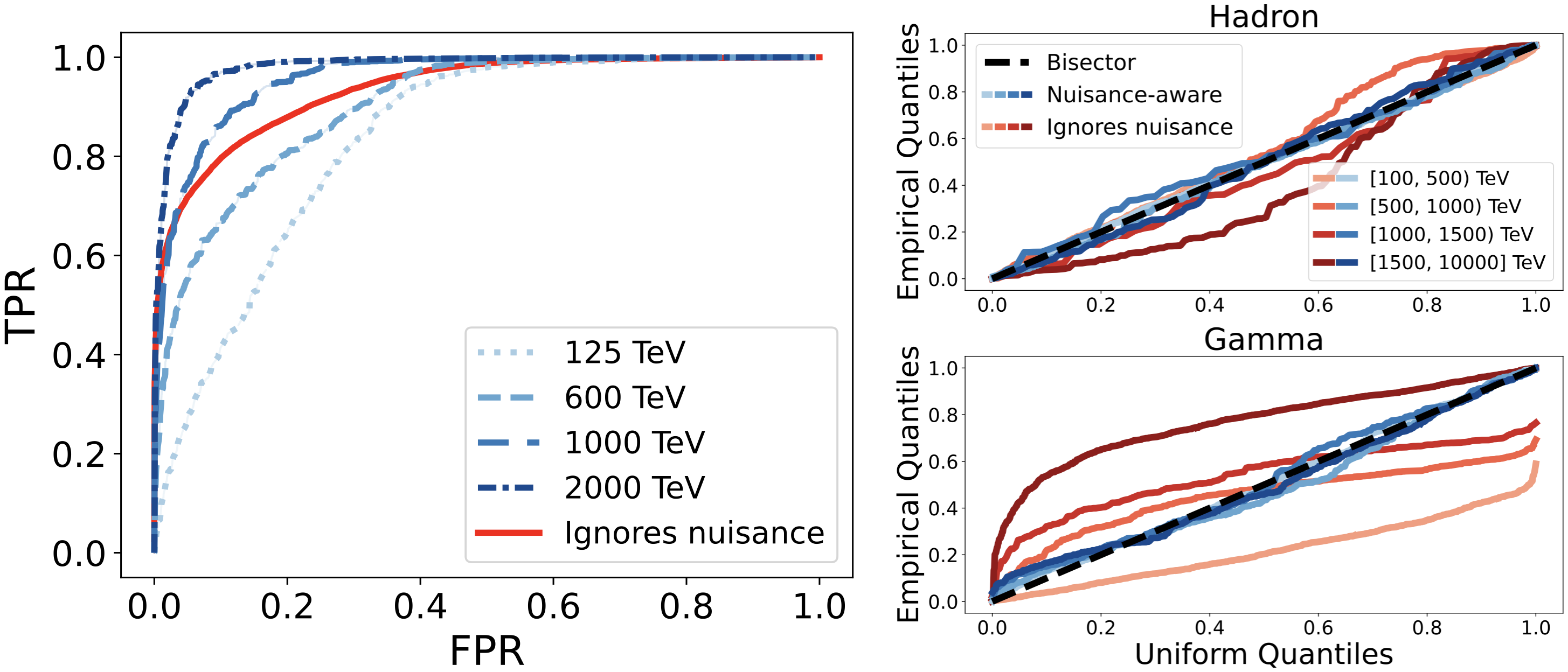}
    \caption{\textbf{Dependence of the ROC on the energy of the cosmic-ray shower.} 
    \textit{Left:} Receiver operating characteristic evaluated according to our method at different energy values (shades of \textcolor{NavyBlue}{blue}). By estimating the entire ROC, we can control FPR or TPR at specified confidence levels for all $\bnu \in \Ncal$, which is not possible with the “marginal” ROC curve (\textcolor{BrickRed}{red}). \textit{Right:} Diagnostic P-P plot evaluated at four bins over energy for nuisance-aware ROC (shades of \textcolor{NavyBlue}{blue}) and ROC that ignores nuisances (shades of \textcolor{BrickRed}{red}). To check if $\pr_\target \left(  \lambda(\X) \leq C |y,\bnu\right)$ is well estimated, we plot PIT values against a $\mathcal{U}(0, 1)$ distribution (dashed bisector; see Appendix~\ref{sec:roc_diag} for details). This is clearly not the case if one ignores nuisance parameters.}
    \label{fig:cosmic_ROC}
\end{figure}

High-energy cosmic rays, both charged and neutral, are extremely informative probes of astrophysical sources in our galaxy and beyond. Gamma rays (which constitute the vast majority of neutral cosmics) reach the Earth atmosphere from specific directions that coincide with the location of the originating source in the sky. On the other hand, charged cosmic rays (hadrons) arrive from non-informative directions as they get deflected by galactic magnetic fields while travelling. An important step in analyzing gamma-ray sources is to separate gamma-induced showers (G) from the very large background ($>1000:1$) of hadron-induced showers (H) using ground-based detector arrays that collect particles $\x$ from secondary showers (\citet{dorigo2023end}; see top left of Figure~\ref{fig:confset_teststat_data} for an illustration). G/H separation is a challenging rare-event detection problem, where the true distribution of both the shower type $Y$ and the shower parameters $\bnu$ might be misspecified in simulated data. Our goal is to infer the cosmic ray identity $Y$ from ground measurements $\X$ while accounting for additional shower parameters: energy $E$, azimuth angle $A$ and zenith angle $Z$. Together, these form a nuisance parameter vector $\bnu = (E, A, Z)$. We construct a data set of $99{,}850$ samples simulated from \texttt{CORSIKA} \citep{heck1998corsika} divided into train ($45\%$), calibration ($45\%$) and test ($10\%$) sets. Figure~\ref{fig:cosmic_ROC} (left) shows several ROC curves as a function of different energy values, demonstrating a clear dependency of the classification problem on this shower parameter.

\begin{figure}[t!]
    \centering
    \includegraphics[width=1\columnwidth]{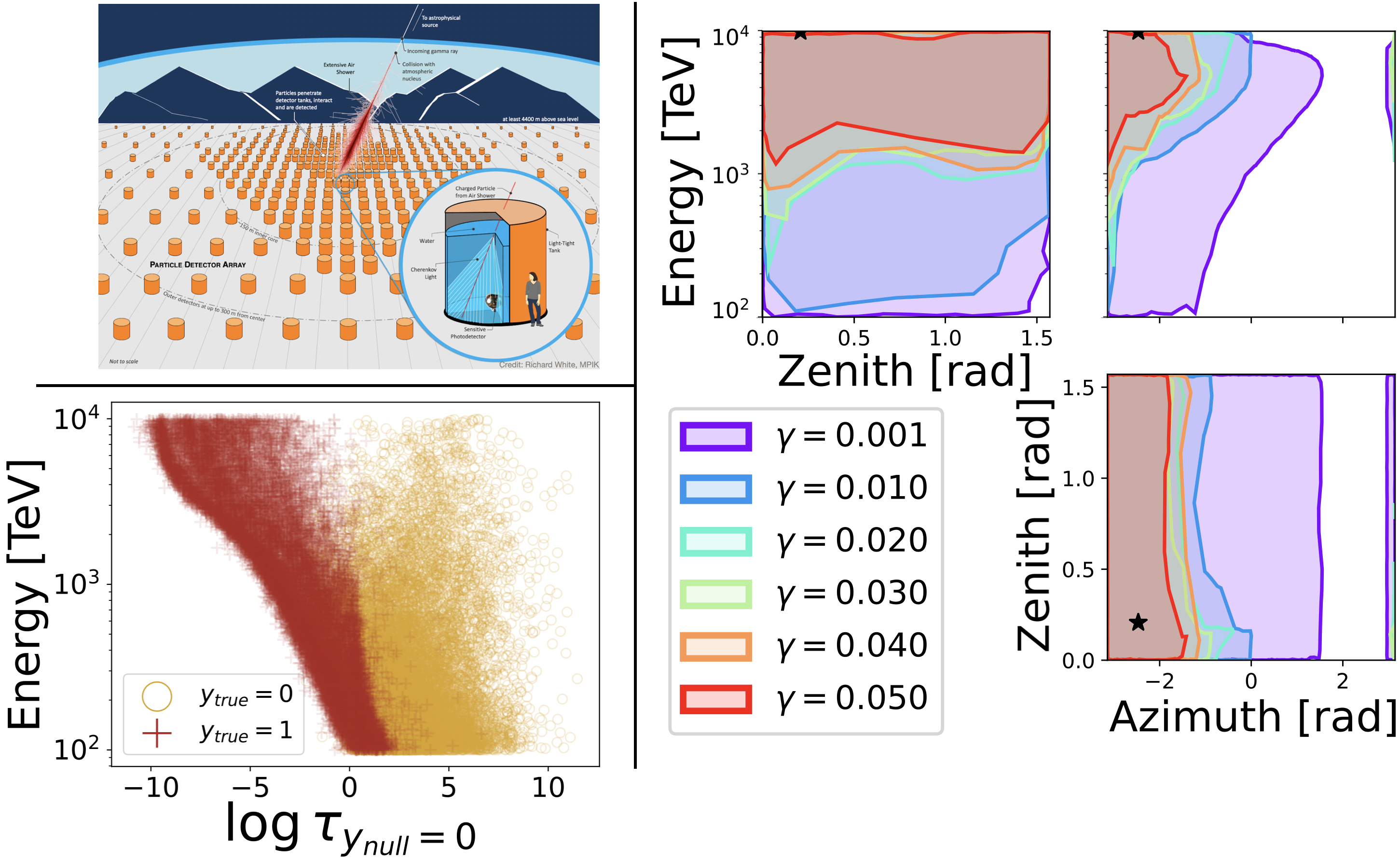}
    \caption{\textbf{Constraining the cosmic ray shower parameters.} \textit{Top left}: Illustration of the Southern Wide-field Gamma-ray Observatory (SWGO; \citet{abreu2019southern}; image credit: Richard White) array of detectors with an incoming gamma ray (\textcolor{BrickRed}{red}). \textit{Bottom Left}: Test statistic under $y_0 = 0$ (hadron) as a function of energy. At high energies, the class-conditional test statistics are well separated, implying that it is easier to distinguish gamma showers (\textcolor{BrickRed}{red}) from hadron showers (\textcolor{YellowOrange}{gold}). \textit{Right}: Confidence set for $\bnu$ at different $(1-\gamma)$ confidence levels obtained via the framework of \citet{masserano2023simulator}. The true value of $\bnu$ is the black star.}
    \label{fig:confset_teststat_data}
\end{figure}

Figure~\ref{fig:cosmic_proportions_gamma} summarizes our results as a function of the confidence level $(1-\alpha)$ for different classification metrics. These are computed within true and within predicted gamma rays for two different bins whose border is the median energy level. Nuisance-aware prediction sets (NAPS with $\gamma=0$) achieve high precision and low false discovery rates but slightly under-perform relative to the standard Bayes classifier (Appendix~\ref{app:bayes_clf}) for lower energy values (left column in Figure~\ref{fig:cosmic_proportions_gamma}), specifically at low confidence levels. This behaviour originates from the complexity of the data: at lower energies it is indeed much harder to distinguish gamma rays from hadrons (see bottom left panel of Figure~\ref{fig:confset_teststat_data}). 
\begin{figure}[b!]
    \centering
    \includegraphics[width=1\columnwidth]{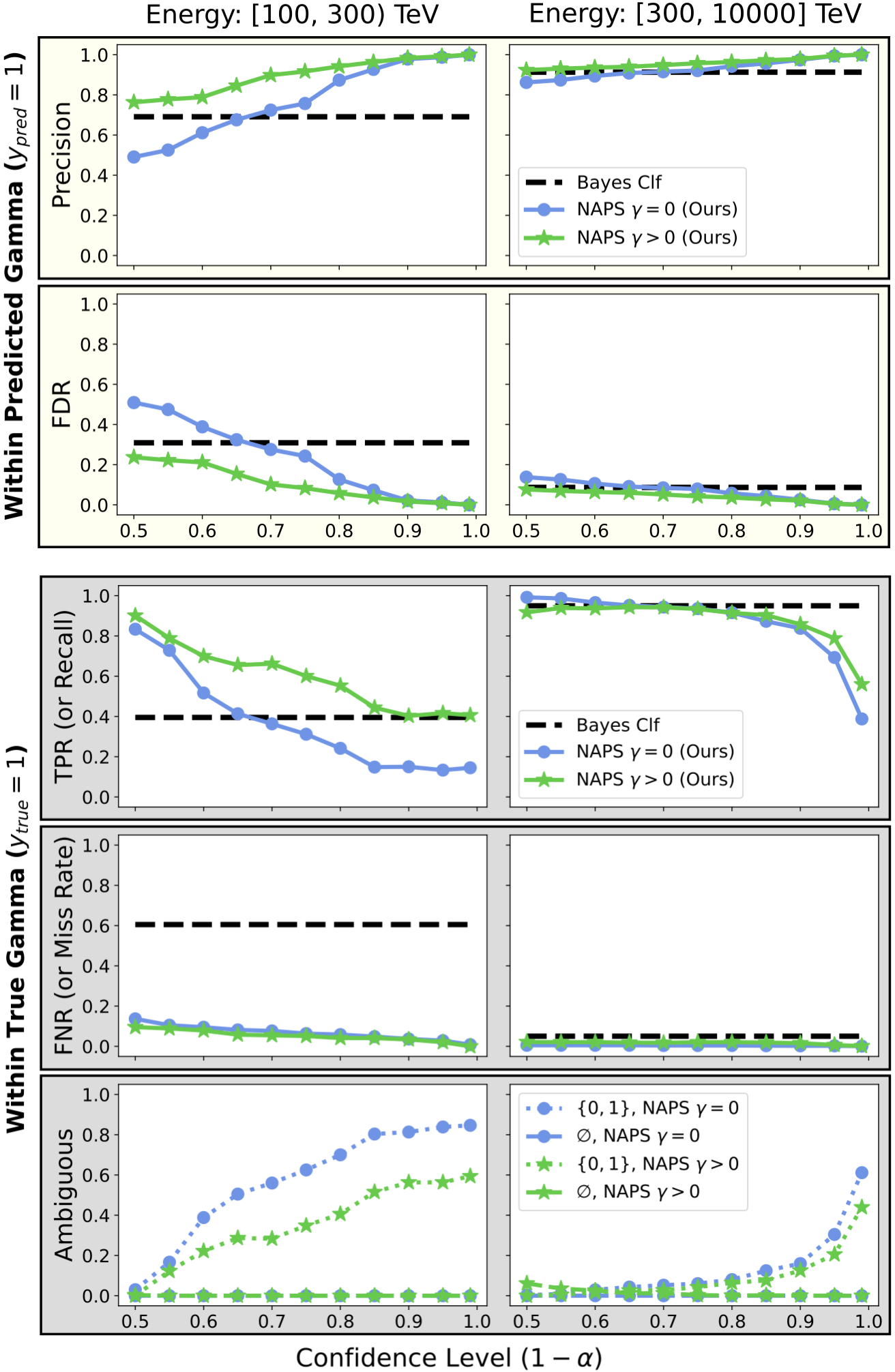}
    \caption{\textbf{Classification metrics within true and within predicted Gamma rays ($y=1$).} Results are binned according to whether the shower energy is below (left) or above (right) the median value. \textit{Top panel:} Nuisance-aware prediction sets (NAPS $\gamma=0$; \textcolor{NavyBlue}{blue}) achieve high precision and low false discovery rates (FDR), especially at high confidence levels. In addition, by constraining the nuisance parameters $\bnu = (E, A, Z)$, we can increase performance (NAPS $\gamma>0$; \textcolor{LimeGreen}{green}) with uniformly better results relative to the standard Bayes classifier (black dashed line). \textit{Bottom panel:} Our set-valued classifier makes explicit its level of uncertainty on the label $y$ by returning ambiguous prediction sets (bottom row) for hard-to-classify $\x_{\target}$. Even so, NAPS with $\gamma>0$ is able to achieve a higher number of true positives and lower number of false negatives relative to the Bayes classifier. Here $\gamma = \alpha \times 0.3$.}
    \label{fig:cosmic_proportions_gamma}
\end{figure}

By constructing $(1-\gamma)$ confidence sets for $\bnu$ (see the right panel of Figure~\ref{fig:confset_teststat_data} for an example), we are able to outperform the standard Bayes classifier at all confidence levels (NAPS with $\gamma>0$). This result is explained by the bottom panel in Figure~\ref{fig:cosmic_proportions_gamma}: NAPS predicts a single label only when it is relatively certain about it, and otherwise outputs an ambiguous prediction set that contains both labels. Nonetheless, for this example, NAPS with $\gamma=0$ is able to achieve a higher number of true positives and lower number of false negatives relative to the Bayes classifier. Additional results and details on the models used can be found in Appendix~\ref{app:cosmic_rays}.
\section{Conclusion and Discussion}\label{sec:conclusion}

The introduction of nuisance parameters complicates the effectiveness and reliability of machine learning models in tasks such as classification. This paper introduces a new method for handling prior probability shift of both label and nuisance parameters in likelihood-free inference when a high-fidelity mechanistic model is available. We demonstrate a new technique for estimating the ROC across the entire parameter space for binary classification problems. We also show how to construct set-valued classifiers that have a guaranteed user-specified probability $(1-\alpha)$ of including the true label (parameter of interest), for all levels $\alpha \in [0,1]$ simultaneously, without having to retrain the model for every $\alpha$. These set-valued classifiers are valid, no matter what the true label and unknown nuisance parameters are. Finally, we demonstrate how to increase power while maintaining validity by constraining nuisance parameters.

{\bf Extensions and Limitations.}  Our approach can be extended to standard classification problems where the training data does not come from a simulator, as long as \add{(i) the nuisance parameters $\bnu$ in the data-generating process have been identified and are available at training time, and (ii) we can reliably estimate the rejection probability function across the entire parameter space as in Section~\ref{sec:RejectionAndROC}. We recommend checking the latter with diagnostic P-P plots (see Appendix~\ref{sec:roc_diag}, and Figure~\ref{fig:cosmic_ROC} (right) for an example).} 

NAPS directly extends to \add{multiclass as one-vs-one} problems, since we can estimate one-vs-one ROC curves for each $\bnu \in \mathcal{N}$. \add{The computational cost for $K$ classes would increase by a factor of $\binom{K}{2}$. However, an extension to multiclass as one-vs-rest problems is non-trivial, because estimating ROC curves requires knowledge of the distribution of labels $Y$ on the target set for every nuisance parameter $\mathbf{\nu}$. Without such knowledge, the ROC curves would not be invariant to GLS.}

\add{NAPS achieves validity under GLS. However, in the absence of a shift, this results in reduced power compared to standard prediction sets (Equation~\ref{eq:standardPS}).}
Although we can recover some of this power by constraining nuisance parameters (i.e. setting $\gamma > 0$), the cutoffs need to be computed for {\em each} test point, which can be computationally expensive, especially for high-dimensional $\bnu$. Furthermore, setting $\gamma > 0$ is not guaranteed to increase power relative to  $\gamma = 0$: Since rejection probability inversion is performed at level $\alpha - \gamma$, power might $\emph{decrease}$ when optimizing the NAPS cutoff over the ($1 - \gamma$) confidence set for $\bnu$ (see Equation \ref{eq:cutoff_FPR}). This can occur if the ($1 - \gamma$) confidence sets are too large, or when the distribution of $\bnu$ is skewed toward certain regions (Figure \ref{fig:synth_nucs}). For further discussion, refer to Appendix ~\ref{sec:increase_power}.

Finally, we note that NAPS may sometimes result in empty prediction sets, though this is uncommon when $(1-\alpha)$ is large. Future adaptations could incorporate strategies from \citet{sadinle2019least} to mitigate this issue.
\section*{Acknowledgements} We thank Larry Wasserman and Mikael Kuusela for their feedback on earlier versions of the work. We are also grateful to Federico Nardi and Will Townes for their help with the cosmic rays and RNA sequencing examples, respectively. ABL \add{and AS are} partially supported by NSF DMS-2053804. RI is partially supported by FAPESP-2023/07068-1 and CNPq-305065/2023-8.
\section*{Impact Statement} In the physical and biological sciences, nuisance parameters are often needed to account for limitations in the modeling of the underlying processes. However, their inclusion reduces the effectiveness of machine learning and statistical procedures. Nuisance parameters (sometimes also referred to as systematic uncertainties) are one of the main factors limiting the precision and discovery reach of scientific analyses. Our work addresses this issue and could have a broader impact on reliable scientific discovery.

\bibliography{references}

\begin{thebibliography}{50}
\providecommand{\natexlab}[1]{#1}
\providecommand{\url}[1]{\texttt{#1}}
\expandafter\ifx\csname urlstyle\endcsname\relax
  \providecommand{\doi}[1]{doi: #1}\else
  \providecommand{\doi}{doi: \begingroup \urlstyle{rm}\Url}\fi

\bibitem[Abreu et~al.(2019)Abreu, Albert, Alfaro, Alvarez, Arceo, Assis, Barao, Bazo, Beacom, Bellido, et~al.]{abreu2019southern}
Abreu, P., Albert, A., Alfaro, R., Alvarez, C., Arceo, R., Assis, P., Barao, F., Bazo, J., Beacom, J., Bellido, J., et~al.
\newblock The southern wide-field gamma-ray observatory (swgo): A next-generation ground-based survey instrument for vhe gamma-ray astronomy.
\newblock \emph{arXiv preprint arXiv:1907.07737}, 2019.

\bibitem[Berger \& Boos(1994)Berger and Boos]{berger1994p}
Berger, R.~L. and Boos, D.~D.
\newblock P values maximized over a confidence set for the nuisance parameter.
\newblock \emph{Journal of the American Statistical Association}, 89\penalty0 (427):\penalty0 1012--1016, 1994.

\bibitem[Brehmer et~al.(2020)Brehmer, Louppe, Pavez, and Cranmer]{brehmer2020mining}
Brehmer, J., Louppe, G., Pavez, J., and Cranmer, K.
\newblock Mining gold from implicit models to improve likelihood-free inference.
\newblock \emph{Proceedings of the National Academy of Sciences}, 117\penalty0 (10):\penalty0 5242--5249, 2020.

\bibitem[Brent(2013)]{brent2013algorithms}
Brent, R.~P.
\newblock \emph{Algorithms for minimization without derivatives}.
\newblock Courier Corporation, 2013.

\bibitem[Chuang \& Lai(1998)Chuang and Lai]{chuang1998hybridresampling}
Chuang, C.-S. and Lai, T.~L.
\newblock {Resampling methods for confidence intervals in group sequential trials}.
\newblock \emph{Biometrika}, 85\penalty0 (2):\penalty0 317--332, 06 1998.
\newblock ISSN 0006-3444.
\newblock \doi{10.1093/biomet/85.2.317}.
\newblock URL \url{https://doi.org/10.1093/biomet/85.2.317}.

\bibitem[Cook et~al.(2006)Cook, Gelman, and Rubin]{cook2006validation}
Cook, S.~R., Gelman, A., and Rubin, D.~B.
\newblock Validation of software for bayesian models using posterior quantiles.
\newblock \emph{Journal of Computational and Graphical Statistics}, 15\penalty0 (3):\penalty0 675--692, 2006.

\bibitem[Cousins(2006)]{cousins2006treatment}
Cousins, R.~D.
\newblock Treatment of nuisance parameters in high energy physics, and possible justifications and improvements in the statistics literature.
\newblock In \emph{Statistical Problems In Particle Physics, Astrophysics And Cosmology}, pp.\  75--85. World Scientific, 2006.

\bibitem[Cowan et~al.(2011)Cowan, Cranmer, Gross, and Vitells]{cowan2011asymptotic}
Cowan, G., Cranmer, K., Gross, E., and Vitells, O.
\newblock Asymptotic formulae for likelihood-based tests of new physics.
\newblock \emph{The European Physical Journal C}, 71:\penalty0 1--19, 2011.

\bibitem[Cranmer et~al.(2020)Cranmer, Brehmer, and Louppe]{cranmer2020frontier}
Cranmer, K., Brehmer, J., and Louppe, G.
\newblock The frontier of simulation-based inference.
\newblock \emph{Proceedings of the National Academy of Sciences}, 117\penalty0 (48):\penalty0 30055--30062, 2020.

\bibitem[Dalmasso et~al.(2021)Dalmasso, Masserano, Zhao, Izbicki, and Lee]{dalmasso2021likelihood}
Dalmasso, N., Masserano, L., Zhao, D., Izbicki, R., and Lee, A.~B.
\newblock Likelihood-free frequentist inference: Confidence sets with correct conditional coverage.
\newblock \emph{arXiv preprint arXiv:2107.03920}, 2021.

\bibitem[Dey et~al.(2022)Dey, Zhao, Newman, Andrews, Izbicki, and Lee]{dey2022calibrated}
Dey, B., Zhao, D., Newman, J.~A., Andrews, B.~H., Izbicki, R., and Lee, A.~B.
\newblock Calibrated predictive distributions via diagnostics for conditional coverage.
\newblock \emph{arXiv preprint arXiv:2205.14568}, 2022.

\bibitem[Ding et~al.(2019)Ding, Adiconis, Simmons, Kowalczyk, Hession, Marjanovic, Hughes, Wadsworth, Burks, Nguyen, et~al.]{ding2019systematic}
Ding, J., Adiconis, X., Simmons, S.~K., Kowalczyk, M.~S., Hession, C.~C., Marjanovic, N.~D., Hughes, T.~K., Wadsworth, M.~H., Burks, T., Nguyen, L.~T., et~al.
\newblock Systematic comparative analysis of single cell rna-sequencing methods.
\newblock \emph{BioRxiv}, pp.\  632216, 2019.

\bibitem[Dorigo \& de~Castro(2020)Dorigo and de~Castro]{dorigo2020dealing}
Dorigo, T. and de~Castro, P.
\newblock Dealing with nuisance parameters using machine learning in high energy physics: a review.
\newblock \emph{arXiv preprint arXiv:2007.09121}, 2020.

\bibitem[Dorigo et~al.(2023)Dorigo, Aehle, Donini, Doro, Gauger, Izbicki, Lee, Masserano, Nardi, Shen, et~al.]{dorigo2023end}
Dorigo, T., Aehle, M., Donini, J., Doro, M., Gauger, N.~R., Izbicki, R., Lee, A., Masserano, L., Nardi, F., Shen, A., et~al.
\newblock End-to-end optimization of the layout of a gamma ray observatory.
\newblock \emph{arXiv preprint arXiv:2310.01857}, 2023.

\bibitem[D’Isanto \& Polsterer(2018)D’Isanto and Polsterer]{d2018photometric}
D’Isanto, A. and Polsterer, K.~L.
\newblock Photometric redshift estimation via deep learning-generalized and pre-classification-less, image based, fully probabilistic redshifts.
\newblock \emph{Astronomy \& Astrophysics}, 609:\penalty0 A111, 2018.

\bibitem[Fawcett \& Flach(2005)Fawcett and Flach]{fawcett2005response}
Fawcett, T. and Flach, P.~A.
\newblock A response to webb and ting’s on the application of roc analysis to predict classification performance under varying class distributions.
\newblock \emph{Machine Learning}, 58:\penalty0 33--38, 2005.

\bibitem[Feldman \& Cousins(1998)Feldman and Cousins]{Feldman1998UnifyingApproach}
Feldman, G.~J. and Cousins, R.~D.
\newblock Unified approach to the classical statistical analysis of small signals.
\newblock \emph{Physical Review D}, 57\penalty0 (7):\penalty0 3873–3889, Apr 1998.
\newblock ISSN 1089-4918.
\newblock \doi{10.1103/physrevd.57.3873}.

\bibitem[Freeman et~al.(2017)Freeman, Izbicki, and Lee]{freeman2017unified}
Freeman, P.~E., Izbicki, R., and Lee, A.~B.
\newblock A unified framework for constructing, tuning and assessing photometric redshift density estimates in a selection bias setting.
\newblock \emph{Monthly Notices of the Royal Astronomical Society}, 468\penalty0 (4):\penalty0 4556--4565, 2017.

\bibitem[Heck et~al.(1998)Heck, Knapp, Capdevielle, Schatz, Thouw, et~al.]{heck1998corsika}
Heck, D., Knapp, J., Capdevielle, J., Schatz, G., Thouw, T., et~al.
\newblock Corsika: A monte carlo code to simulate extensive air showers.
\newblock \emph{Report fzka}, 6019\penalty0 (11), 1998.

\bibitem[{HEP ML Community}()]{hepmllivingreview}
{HEP ML Community}.
\newblock {A Living Review of Machine Learning for Particle Physics}.
\newblock URL \url{https://iml-wg.github.io/HEPML-LivingReview/}.

\bibitem[Hermans et~al.(2021)Hermans, Delaunoy, Rozet, Wehenkel, and Louppe]{hermans2021averting}
Hermans, J., Delaunoy, A., Rozet, F., Wehenkel, A., and Louppe, G.
\newblock Averting a crisis in simulation-based inference.
\newblock \emph{stat}, 1050:\penalty0 14, 2021.

\bibitem[Izbicki et~al.(2017)Izbicki, Lee, and Freeman]{izbicki2017photo}
Izbicki, R., Lee, A.~B., and Freeman, P.~E.
\newblock Photo-z estimation: An example of nonparametric conditional density estimation under selection bias.
\newblock \emph{The Annals of Applied Statistics}, 2017.

\bibitem[Kitching et~al.(2009)Kitching, Amara, Abdalla, Joachimi, and Refregier]{kitching2009cosmological}
Kitching, T., Amara, A., Abdalla, F., Joachimi, B., and Refregier, A.
\newblock Cosmological systematics beyond nuisance parameters: form-filling functions.
\newblock \emph{Monthly Notices of the Royal Astronomical Society}, 399\penalty0 (4):\penalty0 2107--2128, 2009.

\bibitem[Lei et~al.(2018)Lei, G'Sell, Rinaldo, Tibshirani, and Wasserman]{Lei2018}
Lei, J., G'Sell, M., Rinaldo, A., Tibshirani, R.~J., and Wasserman, L.
\newblock Distribution-free predictive inference for regression.
\newblock \emph{Journal of the American Statistical Association}, 113\penalty0 (523):\penalty0 1094--1111, 2018.

\bibitem[Lipton et~al.(2018)Lipton, Wang, and Smola]{lipton2018detecting}
Lipton, Z., Wang, Y.-X., and Smola, A.
\newblock Detecting and correcting for label shift with black box predictors.
\newblock In \emph{International conference on machine learning}, pp.\  3122--3130. PMLR, 2018.

\bibitem[Louppe et~al.(2017)Louppe, Kagan, and Cranmer]{louppe2017learning}
Louppe, G., Kagan, M., and Cranmer, K.
\newblock Learning to pivot with adversarial networks.
\newblock \emph{Advances in neural information processing systems}, 30, 2017.

\bibitem[Masserano et~al.(2023)Masserano, Dorigo, Izbicki, Kuusela, and Lee]{masserano2023simulator}
Masserano, L., Dorigo, T., Izbicki, R., Kuusela, M., and Lee, A.
\newblock Simulator-based inference with {WALDO}: Confidence regions by leveraging prediction algorithms and posterior estimators for inverse problems.
\newblock In \emph{International Conference on Artificial Intelligence and Statistics}, pp.\  2960--2974. PMLR, 2023.

\bibitem[Moreno-Torres et~al.(2012)Moreno-Torres, Raeder, Alaiz-Rodr{\'\i}guez, Chawla, and Herrera]{moreno2012unifying}
Moreno-Torres, J.~G., Raeder, T., Alaiz-Rodr{\'\i}guez, R., Chawla, N.~V., and Herrera, F.
\newblock A unifying view on dataset shift in classification.
\newblock \emph{Pattern recognition}, 45\penalty0 (1):\penalty0 521--530, 2012.

\bibitem[Papadopoulos et~al.(2002)Papadopoulos, Proedrou, Vovk, and Gammerman]{papadopoulos2002inductive}
Papadopoulos, H., Proedrou, K., Vovk, V., and Gammerman, A.
\newblock Inductive confidence machines for regression.
\newblock In \emph{European Conference on Machine Learning}, pp.\  345--356. Springer, 2002.

\bibitem[Papamakarios et~al.(2017)Papamakarios, Pavlakou, and Murray]{papamakarios2017masked}
Papamakarios, G., Pavlakou, T., and Murray, I.
\newblock Masked autoregressive flow for density estimation.
\newblock \emph{Advances in neural information processing systems}, 30, 2017.

\bibitem[Podkopaev \& Ramdas(2021)Podkopaev and Ramdas]{podkopaev2021distribution}
Podkopaev, A. and Ramdas, A.
\newblock Distribution-free uncertainty quantification for classification under label shift.
\newblock In \emph{Uncertainty in Artificial Intelligence}, pp.\  844--853. PMLR, 2021.

\bibitem[Polo et~al.(2023)Polo, Izbicki, Lacerda~Jr, Ibieta-Jimenez, and Vicente]{polo2023unified}
Polo, F.~M., Izbicki, R., Lacerda~Jr, E.~G., Ibieta-Jimenez, J.~P., and Vicente, R.
\newblock A unified framework for dataset shift diagnostics.
\newblock \emph{Information Sciences}, pp.\  119612, 2023.

\bibitem[Pouget et~al.(2013)Pouget, Beck, Ma, and Latham]{pouget2013probabilistic}
Pouget, A., Beck, J.~M., Ma, W.~J., and Latham, P.~E.
\newblock Probabilistic brains: knowns and unknowns.
\newblock \emph{Nature neuroscience}, 16\penalty0 (9):\penalty0 1170--1178, 2013.

\bibitem[Prokhorenkova et~al.(2018)Prokhorenkova, Gusev, Vorobev, Dorogush, and Gulin]{prokhorenkova2018catboost}
Prokhorenkova, L., Gusev, G., Vorobev, A., Dorogush, A.~V., and Gulin, A.
\newblock Catboost: unbiased boosting with categorical features.
\newblock \emph{Advances in neural information processing systems}, 31, 2018.

\bibitem[Quinonero-Candela et~al.(2008)Quinonero-Candela, Sugiyama, Schwaighofer, and Lawrence]{quinonero2008dataset}
Quinonero-Candela, J., Sugiyama, M., Schwaighofer, A., and Lawrence, N.~D.
\newblock \emph{Dataset shift in machine learning}.
\newblock Mit Press, 2008.

\bibitem[Rizvi et~al.(2023)Rizvi, Pettee, and Nachman]{rizvi2023learning}
Rizvi, S., Pettee, M., and Nachman, B.
\newblock Learning likelihood ratios with neural network classifiers.
\newblock \emph{arXiv preprint arXiv:2305.10500}, 2023.

\bibitem[Romano et~al.(2020)Romano, Sesia, and Candes]{romano2020classification}
Romano, Y., Sesia, M., and Candes, E.
\newblock Classification with valid and adaptive coverage.
\newblock \emph{Advances in Neural Information Processing Systems}, 33:\penalty0 3581--3591, 2020.

\bibitem[Sadinle et~al.(2019)Sadinle, Lei, and Wasserman]{sadinle2019least}
Sadinle, M., Lei, J., and Wasserman, L.
\newblock Least ambiguous set-valued classifiers with bounded error levels.
\newblock \emph{Journal of the American Statistical Association}, 114\penalty0 (525):\penalty0 223--234, 2019.

\bibitem[Saerens et~al.(2002)Saerens, Latinne, and Decaestecker]{saerens2002adjusting}
Saerens, M., Latinne, P., and Decaestecker, C.
\newblock Adjusting the outputs of a classifier to new a priori probabilities: a simple procedure.
\newblock \emph{Neural computation}, 14\penalty0 (1):\penalty0 21--41, 2002.

\bibitem[Sen et~al.(2009)Sen, Walker, and Woodroofe]{Sen2009NuisanceParameters}
Sen, B., Walker, M., and Woodroofe, M.
\newblock On the unified method with nuisance parameters.
\newblock \emph{Statistica Sinica}, 19\penalty0 (1):\penalty0 301--314, 2009.
\newblock ISSN 10170405, 19968507.

\bibitem[Song et~al.(2023)Song, Wang, Yan, Liu, Sun, and Li]{song2023scdesign3}
Song, D., Wang, Q., Yan, G., Liu, T., Sun, T., and Li, J.~J.
\newblock scdesign3 generates realistic in silico data for multimodal single-cell and spatial omics.
\newblock \emph{Nature Biotechnology}, pp.\  1--6, 2023.

\bibitem[Storkey et~al.(2009)]{storkey2009training}
Storkey, A. et~al.
\newblock When training and test sets are different: characterizing learning transfer.
\newblock \emph{Dataset shift in machine learning}, 30\penalty0 (3-28):\penalty0 6, 2009.

\bibitem[Storn \& Price(1997)Storn and Price]{storn1997differential}
Storn, R. and Price, K.
\newblock Differential evolution--a simple and efficient heuristic for global optimization over continuous spaces.
\newblock \emph{Journal of global optimization}, 11:\penalty0 341--359, 1997.

\bibitem[Tibshirani et~al.(2019)Tibshirani, Foygel~Barber, Candes, and Ramdas]{tibshirani2019conformal}
Tibshirani, R.~J., Foygel~Barber, R., Candes, E., and Ramdas, A.
\newblock Conformal prediction under covariate shift.
\newblock \emph{Advances in neural information processing systems}, 32, 2019.

\bibitem[Vaz et~al.(2019)Vaz, Izbicki, and Stern]{vaz2019quantification}
Vaz, A.~F., Izbicki, R., and Stern, R.~B.
\newblock Quantification under prior probability shift: The ratio estimator and its extensions.
\newblock \emph{The Journal of Machine Learning Research}, 20\penalty0 (1):\penalty0 2921--2953, 2019.

\bibitem[Virtanen et~al.(2020)Virtanen, Gommers, Oliphant, Haberland, Reddy, Cournapeau, Burovski, Peterson, Weckesser, Bright, et~al.]{virtanen2020scipy}
Virtanen, P., Gommers, R., Oliphant, T.~E., Haberland, M., Reddy, T., Cournapeau, D., Burovski, E., Peterson, P., Weckesser, W., Bright, J., et~al.
\newblock Scipy 1.0: fundamental algorithms for scientific computing in python.
\newblock \emph{Nature methods}, 17\penalty0 (3):\penalty0 261--272, 2020.

\bibitem[Vovk et~al.(2014)Vovk, Petej, and Fedorova]{vovk2014conformal}
Vovk, V., Petej, I., and Fedorova, V.
\newblock From conformal to probabilistic prediction.
\newblock In \emph{Artificial Intelligence Applications and Innovations: AIAI 2014 Workshops: CoPA, MHDW, IIVC, and MT4BD, Rhodes, Greece, September 19-21, 2014. Proceedings 10}, pp.\  221--230. Springer, 2014.

\bibitem[Vovk et~al.(2016)Vovk, Fedorova, Nouretdinov, and Gammerman]{vovk2016criteria}
Vovk, V., Fedorova, V., Nouretdinov, I., and Gammerman, A.
\newblock Criteria of efficiency for conformal prediction.
\newblock In \emph{Conformal and Probabilistic Prediction with Applications: 5th International Symposium, COPA 2016, Madrid, Spain, April 20-22, 2016, Proceedings 5}, pp.\  23--39. Springer, 2016.

\bibitem[Vovk et~al.(2005)]{Vovk2005}
Vovk, V. et~al.
\newblock \emph{Algorithmic learning in a random world}.
\newblock Springer Science \& Business Media, 2005.

\bibitem[Zhao et~al.(2021)Zhao, Dalmasso, Izbicki, and Lee]{zhao2021diagnostics}
Zhao, D., Dalmasso, N., Izbicki, R., and Lee, A.~B.
\newblock Diagnostics for conditional density models and bayesian inference algorithms.
\newblock In \emph{Uncertainty in Artificial Intelligence}, pp.\  1830--1840. PMLR, 2021.

\end{thebibliography}
\bibliographystyle{icml2024}

\newpage
\appendix
\onecolumn
\section{The Bayes Factor as a Frequentist Test Statistic}\label{sec:app_bayes_factor} 

In this work, we treat the Bayes factor as a frequentist test statistic, similar to the Bayes Frequentist Factor (\texttt{BFF}) method in \cite{dalmasso2021likelihood}. Consider the composite-versus-composite hypothesis test: 
 \begin{equation} \label{eq:discrimination_test}
    H_{0,y}: \btheta \in \Theta_0   \ \ \text{versus} \ \ H_{1, y}: \btheta \in \Theta_1 
\end{equation}
where $\Theta_0= \{y\} \times \mathcal{N}$, $\Theta_1= \{y\}^c \times \mathcal{N}$, and $y \in \{0,1\}$.
The Bayes factor of the test
is defined as
\begin{equation*}
 \tau_{y}(\x) := \frac{\P'(\x|H_{0,y})}{\P'(\x|H_{1,y})} =  \frac{\int_{\mathcal{N}}  \mathcal{L}(\x;y, \bnu) \; p'(\bnu|y) \;d\bnu}
 { \int_{\mathcal{N}}   \mathcal{L}(\x; 1-y, \bnu) \;  p'(\bnu|1-y) \; d\bnu} 
\end{equation*}
By Bayes theorem,
\begin{align}
 \tau_{y}(\x) & = 
 \frac{\int_{\mathcal{N}} \frac{ p'(y, \bnu|\x)}{p'(y, \bnu)} 
  p'(\bnu|y) \; d\bnu  }
 { \int_{\mathcal{N}} \frac{ p'( 1-y, \bnu|\x)}{p'( 1-y, \bnu)}   
  p'(\bnu| 1-y) \; d\bnu } \nonumber 
  = \frac{\int_{\mathcal{N}} \frac{ p'(y, \bnu|\x)}
 {\P'(Y=y) } 
  \; d\bnu  }
 { \int_{\mathcal{N}} \frac{ p'( 1-y, \bnu|\x)}{\P'(Y =1- y) }   
   \; d\bnu } \nonumber \\
  &   = \frac{\P'(Y=y|\x) \; \P'(Y =1-y)}
  {\P'(Y =1- y|\x) \; \P'(Y=y)}.
   \label{eq:BFF_statistic}
\end{align}

However, unlike \texttt{BFF}, we are not estimating the likelihood or odds from simulated data, but instead directly evaluate a pretrained classifier $\P'(Y=y|\x)$.

\section{Proofs}\label{sec:poofs}

For simplicity in notation, we will henceforth omit the ``train'' and ``target'' subscripts in $\P$. The symbol $\P'$ will represent the training distribution, while $\P$ will denote the target distribution.

\begin{proof}[Proof of Lemma \ref{lemma:reject_prob_invariance}]
    This follows from the fact that $W_{\lambda}(C; y, \bnu)$ only depends on the conditional randomness of $\X|y,\nu$, which, under GLS, is the same on both train and target data.
\end{proof}

\begin{proof}[Proof of Theorem \ref{lemma:NA_cutoff}] 
Notice that
\begin{equation*}
    \begin{split}
        \P (\lambda(\X) \leq C_{\alpha,y}^*(\X)| y,\bnu ) &= 
        \P(\lambda(\X) \leq C_{\alpha,y}^*(\X),\bnu \in S_y(\X;\gamma)| y,\bnu)
         + \P (\lambda(\X) \leq C_{\alpha,y}^*(\X),\bnu \notin S_y(\X;\gamma)| y,\bnu) \\
        &\leq  \P ( \lambda(\X) \leq W^{-1}_\lambda(\beta; y, \bnu)| y,\bnu )
        +\P(\bnu \notin S_y(\X;\gamma)| y,\bnu) \\ 
        &\leq \beta+\gamma=\alpha,
    \end{split}
\end{equation*}
which proves the first part of the result.
Similarly,
\begin{align*}
    \P(\lambda(\X) \geq \widetilde C^*_{\alpha,y}(\X)|1-y,\bnu) &=  
    \P (\lambda(\X) \geq \widetilde C^*_{\alpha,y},\bnu \in S_{1-y}(\X;\gamma)|1-y,\bnu)
    \P (\lambda(\X) \geq \widetilde C^*_{\alpha,y},\bnu \notin S_{1-y}(\X;\gamma)|1-y,\bnu) \\
    &\leq  \P(\lambda(\X) \geq W^{-1}_\lambda(\beta; 1-y, \bnu)|1-y,\bnu)
    + \P (\bnu \notin S_{1-y}(\X;\gamma)|1-y,\bnu) \\ 
    &\leq  1-\beta+\gamma=1-\alpha,
\end{align*}
and therefore 
$$ \P(\lambda(\X) \leq \widetilde C^*_{\alpha,y}(\X)|1-y,\bnu) \geq \alpha,$$
which concludes the proof.
\end{proof}

\begin{proof}[Proof of Theorem \ref{thm:nacs_coverage}]
   By construction 
\begin{align*}
    \P(Y \in \H(\X;\alpha)|y,\bnu) &=  \P\left(\hat \tau_{y}(\X) > C_{\alpha,y}^*(\X)|y,\bnu\right) \\
    &= 1 -\P\left(\hat \tau_{y}(\X) \leq C_{\alpha,y}^*(\X)|y,\bnu\right) \\
    &\geq 1- \alpha,
\end{align*}
where the last inequality follows from Lemma \ref{lemma:NA_cutoff}. This proves the first statement of the theorem. To prove the second statement, notice that
\begin{align*}
    \P(Y \in \H(\X;\alpha))&=\int \P(Y \in \H(\X;\alpha)|y,\bnu)d \mu(y,\bnu)\\ 
    &\geq \int (1-\alpha)d \mu(y,\bnu) \\
    &=1-\alpha,
\end{align*}
where $\mu(y,\bnu)$ denotes the measure on $(Y,\bnu)$ on the target set.
\end{proof}

\section{Estimating the Rejection Probability Function}\label{sec:power_func}

We learn $W_{\lambda}(C; y, \bnu)$ using a monotone regression that enforces the rejection probability to be a non-decreasing function of $C$. For each point $i=1,\ldots, B^\prime$ in the calibration set $\mathcal{T}'=\{(Y_1, \bnu_1, \X_1),\ldots,(Y_{B^\prime}, \bnu_{B^\prime}, \X_{B^\prime})\}$ drawn from $p_\train(\btheta)  \mathcal{L}(\x;\btheta)$ where $\btheta=(Y,\bnu)$, we sample a set of $K$ cutoffs according to the empirical distribution of the test statistic $\lambda$. Then, we regress the random variable 
\begin{align}
Z_{i,j}:= \I \left( \lambda(\X_i) \leq C_j \right )
\end{align}
on $Y_i$, $\bnu_i$ and $C_{i,j}$ ($=C_j$) using  the ``augmented'' calibration set $\mathcal{T}^{\prime\prime}=\{(Y_i, \bnu_i,C_{i,j},Z_{i,j})\}_{i,j}$, for $i=1,\ldots,B^\prime$ and $j=1,\ldots,K$, where $K$ is the augmentation factor. See Algorithm \ref{alg:power_function} for details.

\begin{algorithm}[h!]
    \caption{\texttt{Learning the Rejection Probability Function}}\label{alg:power_function}
    \textbf{Input: }{\small test statistic $\lambda$; calibration data  $\mathcal{T}'=\{(Y_1, \bnu_1, \X_1),\ldots,(Y_{B^\prime}, \bnu_{B^\prime}, \X_{B^\prime})\}$; sampled cutoffs $G=\{C_1, \ldots, C_K \}$}\\
    \textbf{Output: }{\small Estimate of the rejection probability $W_\lambda(C;y, \bnu)$ for all $C \in G$, $y \in \{0,1\}$ and $\bnu \in \mathcal{N}$}	
    \begin{algorithmic}[1]
        \STATE \codecomment{Learn rejection probability from augmented calibration data $\mathcal{T}''$}
        \STATE Set $\mathcal{T}'' \gets \emptyset$
        \FOR{$i$ in $\{1,...,B^\prime\}$}		    
            \FOR{$j$ in $\{1,...,K\}$} 
                \STATE Compute $Y_{i,j} \gets \I \left(  \lambda(\X_i) \leq C_{j} \right)$
                \STATE Let $\mathcal{T}'' \gets  \mathcal\mathcal{T}''  \cup \{\left(Y_i, \bnu_i,C_{j}, Z_{i,j} \right)\}$ 
            \ENDFOR
        \ENDFOR  
        \STATE Estimate $ W_\lambda(C;y,\bnu) := \pr_{y,\bnu} \left(   \lambda(\X) \leq C \right)$ from $\mathcal{T}^{\prime\prime}$ via a regression of $Z$ on $Y$, $\bnu$ and $C$, which is monotonic in $C$. 
	\STATE \textbf{return} Estimated rejection probabilities $\widehat{W}_\lambda(C;y, \bnu)$, for $C \in G$, $y \in \{0,1\}$ and $\bnu \in \mathcal{N}$ 	
    \end{algorithmic}
\end{algorithm}

\section{Diagnostics of Estimated ROC Curves}\label{sec:roc_diag}

Here we describe how to evaluate goodness-of-fit of an estimate of the rejection probability function. This is inspired by methods that use the Probability Integral Transform (PIT) to assess conditional density estimators \citep{cook2006validation,freeman2017unified,izbicki2017photo,d2018photometric}.

If $W_{\lambda}(C; y, \bnu)=\pr_\target \left(  \lambda(\X) \leq C |y,\bnu\right)=F_{\lambda(\X)|y,\bnu}(C)$ is well estimated, then the random variable $
W_{\lambda}(\lambda(\X'); y, \bnu) \sim U(0,1)$, where $\X'$ is drawn from the simulator using $(y,\bnu)$ as parameters. This suggests we assess the performance of our estimator of $W$, $\widehat W$, via a P-P plot comparing $\widehat W(\lambda(\X_1); Y_1, \bnu_1),\ldots, \widehat W(\lambda(\X_B); Y_B, \bnu_B)$ to a Uniform(0,1) distribution, where $(\lambda(\X_1); Y_1, \bnu_1),\ldots,(\lambda(\X_B); Y_B, \bnu_B)$ denote an evaluation sample drawn from the simulator.  The distribution of these statistics can however be uniform even if $\widehat W$ is not a good estimate \citep[Theorem 1]{zhao2021diagnostics}. Here, we avoid this problem by 
dividing the parameter space ${\Theta}$ into bins and constructing separate distribution plots for samples within each bin.

\section{The Standard Bayes Classifier }\label{app:bayes_clf}

\begin{Lemma}[Bayes classifier] Let $h: \mathcal{X} \rightarrow \{0, 1\}$ be a classification rule. Define the weighted loss
\begin{equation}\label{eq:accuracy}
    W=c_1  I_{\{1\}}(Y) \; I_{\{0\}}(h(\X))+c_0 I_{\{0\}}(Y) \; I_{\{1\}}(h(\X)),
\end{equation}
where $c_k$ is the cost of mis-classifying a $Y=k$ observation, for $k=0, 1$. The  Bayes (that is, optimal) classifier that minimizes the error rate  $\E_{\target}(W)$ averaged over both $\X$ and $Y$ is given by
 \begin{equation}\label{eq:optimal_classifier2}
h^*(x)= 
\begin{cases}
     1 & \text{if  }  \P_\target(Y=1 | \x) > \alpha^*,\\
     0 & \text{if  }  \P_\target(Y=1 | \x) < \alpha^*,\\
    \text{arbitrary}  &\text{if  } \P_\target(Y=1 | \x) = \alpha^*,
\end{cases}  
 \end{equation}
where $\alpha^{*} := \frac{c_0}{c_0  + c_1}.$
\end{Lemma}

\begin{Remark}[Balanced accuracy] If there is no shift between the train and target sets, a common choice for the loss (\ref{eq:accuracy}) is $c_1=1/\P_\train(Y=1)$ and $c_0=1/\P_\train(Y=0)$. This yields the balanced error rate 
$$\E_\train(W) = \P_\train(h(\X) = 0|Y = 1)+ \P_\train(h(\X) = 1|Y = 0)$$ and the cut-off $\alpha^* = \P_\train(Y=1)$ for the Bayes classifier (Equation~\ref{eq:optimal_classifier2}).
\end{Remark}

\begin{Remark}[Bayes classifier under GLS] Under GLS, there is no monotonic relationship between $\P_\target \left(Y=1|\x\right)$ and $\P_\train \left(Y=1|\x\right)$. Thus, it is not possible to use $\P_\train \left(Y=1|\x\right)$ to recover $\P_\target \left(Y=1|\x\right)$ using standard label shift corrections \citep{saerens2002adjusting,lipton2018detecting}. 
\end{Remark}

\begin{Remark}[Bayes classifier under the presence of nuisance parameters but no GLS]
If there is no GLS, $\P_\train \left(Y=1|\x\right) = \P_\target \left(Y=1|\x\right)$. However, without a nuisance-aware cutoff, the Bayes classifier is usually calibrated to control type-I error marginally over $\bnu$. NACS instead controls this error for all $\bnu \in \mathcal{N}$.
\end{Remark}

\section{Additional Results and Details on Cosmic Ray Experiment}\label{app:cosmic_rays}

\subsection{Experimental Set-Up  with Ground-Based Detector Arrays}

The data used in this paper are generated via the CORSIKA cosmic ray simulator \citep{heck1998corsika}. CORSIKA is a Monte Carlo simulation program that models the interactions of primary cosmic rays with the Earth's atmosphere. Given values of the parameters $\mu, E, Z, A$, which define the primary cosmic ray identity, energy, zenith and azimuth angle, respectively, CORSIKA outputs the identities, momenta, positions, and arrival times of all secondary particles generated in the atmospheric shower, that eventually reach the ground and that are mostly muons, electrons and photons at gamma-ray energies with minor abundance of heavier particles. 

The measured data $\x$ in our analysis does not incorporate the full shower footprint, 
as this level of information cannot be captured in any realistic scenario. Instead, we simulate a simple $6\times 6$ detector grid, where each detector covers a $2\times 2$~m$^2$ area, with 48~m detector spacing. Information for a secondary particle of a particular shower footprint is incorporated into the analysis only if that secondary particle lands within the area of a detector. See Figure \ref*{fig:swgo} (right) for a simplified representation of the detector grid.

We assume 100\% detector efficiency and that all secondary particles types are detectable. We also assume that showers always originate at the center of the detector grid. Finally, we assume that both the zenith and azimuth angles $Z$ and $A$ are known due to the relative ease with which they can be estimate from observed footprint data. Thus, our only nuisance parameter for inference on $\mu$ is the energy $E$ of the cosmic ray.

\begin{figure}[t!]
    \centering
    \begin{tabular}{c| c}
        \includegraphics[width=0.45\textwidth]{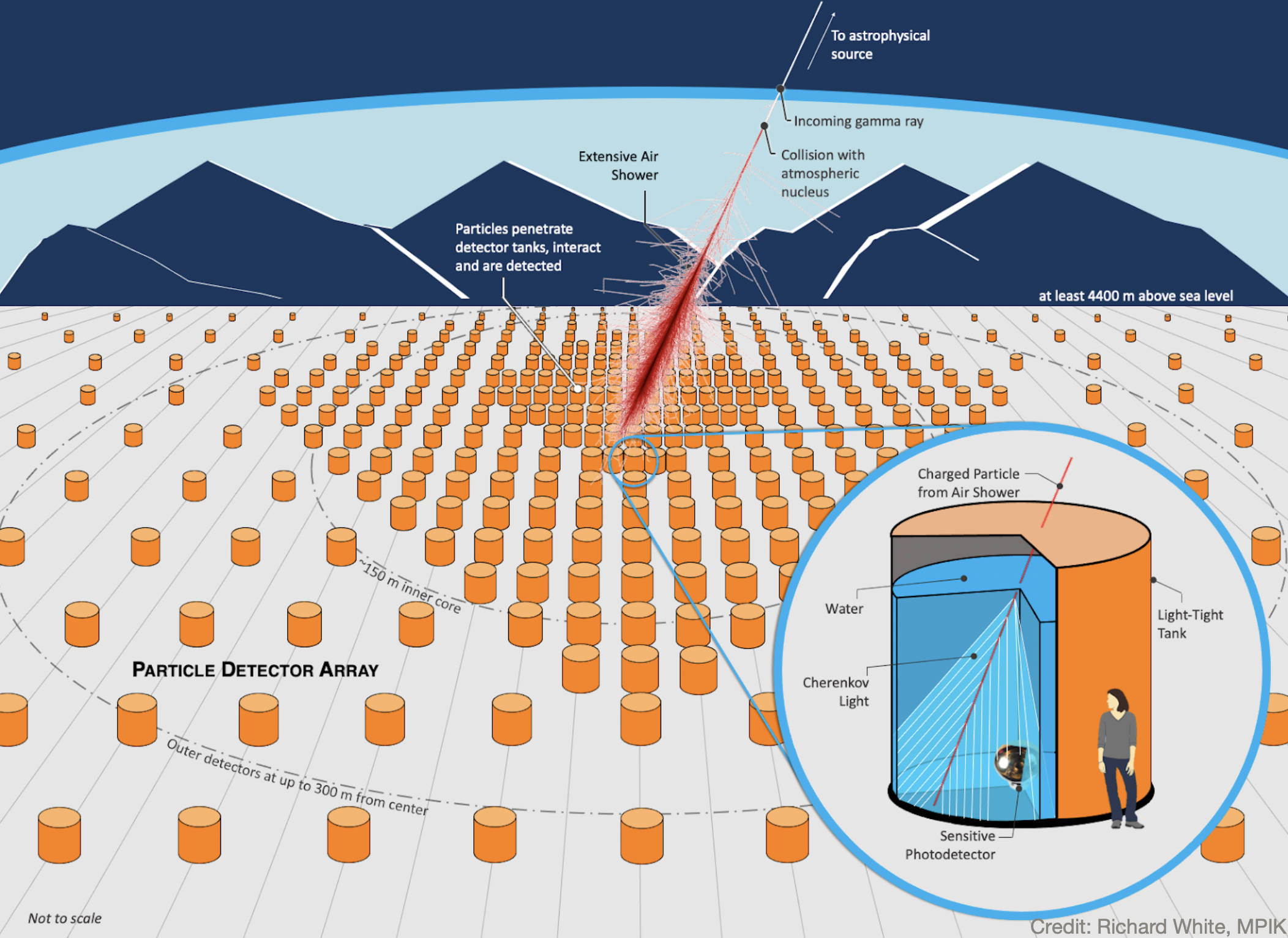} &  
        \includegraphics[width=0.5\textwidth]{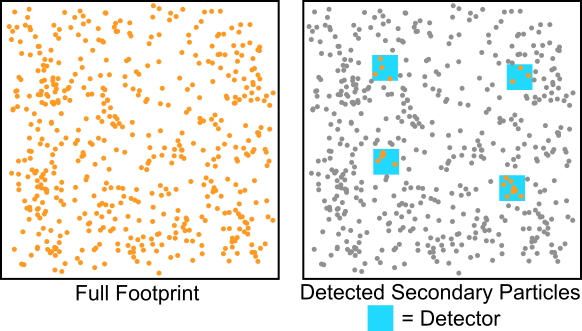}
    \end{tabular}   
    \caption{{\em Left:} Artistic representation of the SWGO array. The inlay shows the individual detector unit. {\em Right:} Although we have access to all secondary particles in our simulated cosmic ray showers, we only include the particles that hit our simulated detector setup (blue rectangles) in the analysis. This layout pictured here is an illustrative example.}
    \label{fig:swgo}
\end{figure}

The data used to estimate the test statistic are drawn according to the following distribution (which may be different from that of actual astrophysical sources):
\begin{enumerate}
  \item Gamma ray to Hadron ratio 1:1 
  (whereas actual observed ratios are in the range 1:1,000 -- 1:100,000)
  \item Energy between 100 TeV and 10 PeV, with probability density proportional to $E^{-1}$ for gamma rays and $E^{-2}$ for hadrons 
  (with standard astrophysical sources closer to between -2:-4)
  \item Zenith uniformly distributed between 0 and 65 degrees 
  \item Azimuth uniformly distributed between -180 and 180 degrees
\end{enumerate}

To derive $\x_i$, we first define four secondary particle groups: photons (neutral); electrons and positrons; muons (charged);
and all other secondary particle types. Then for each simulated detector, we record the count of particles in each group that hit the detector. This results in a vector of length $4 \cdot 36 = 144$ for each primary cosmic ray that represents the detector data. We construct $\x_i$ by concatenating the detector data with $Z_i$ and $A_i$.

For the calibration and test sets, we use the same reference distribution. 

\subsection{Details on the algorithms used in Section~\ref{sec:cosmic_rays}}
We used gradient boosting probabilistic classifiers as implemented in \texttt{CatBoost} \citep{prokhorenkova2018catboost} to estimate both $\mathbb{P}(Y|\X)$ and $W_\lambda(C;y, \bnu)$. For the latter, \texttt{CatBoost} allows to easily enforce monotonicity constraints on the features, which we used on $C$. To compute cutoffs, we used the \texttt{brentq} routine \citep{brent2013algorithms} to calculate the inverse and the differential evolution global optimization algorithm \citep{storn1997differential} to find the infimum. Both are implemented in \texttt{SciPy} \citep{virtanen2020scipy}. To obtain confidence sets for $\bnu$, we used the method developed by \citet{masserano2023simulator} with a masked autoregressive flow \citep{papamakarios2017masked} since it guarantees that the constructed region contains the true value of $\bnu$ at the desired confidence level for all $\bnu \in \Ncal$.

\subsection{Additional Results}

Figures \ref*{fig:cr_within_ph} and \ref*{fig:cr_within_h} mirror the results for \ref*{fig:cosmic_proportions_gamma}, focusing on cosmic rays predicted to be hadrons and true hadron cosmic rays repectively. Identifying hadrons is of lesser scientific value than identifying gamma rays, so the results here are presented mainly for reference.

\begin{figure}[t!]
    \centering
    \includegraphics[width=0.68\textwidth]{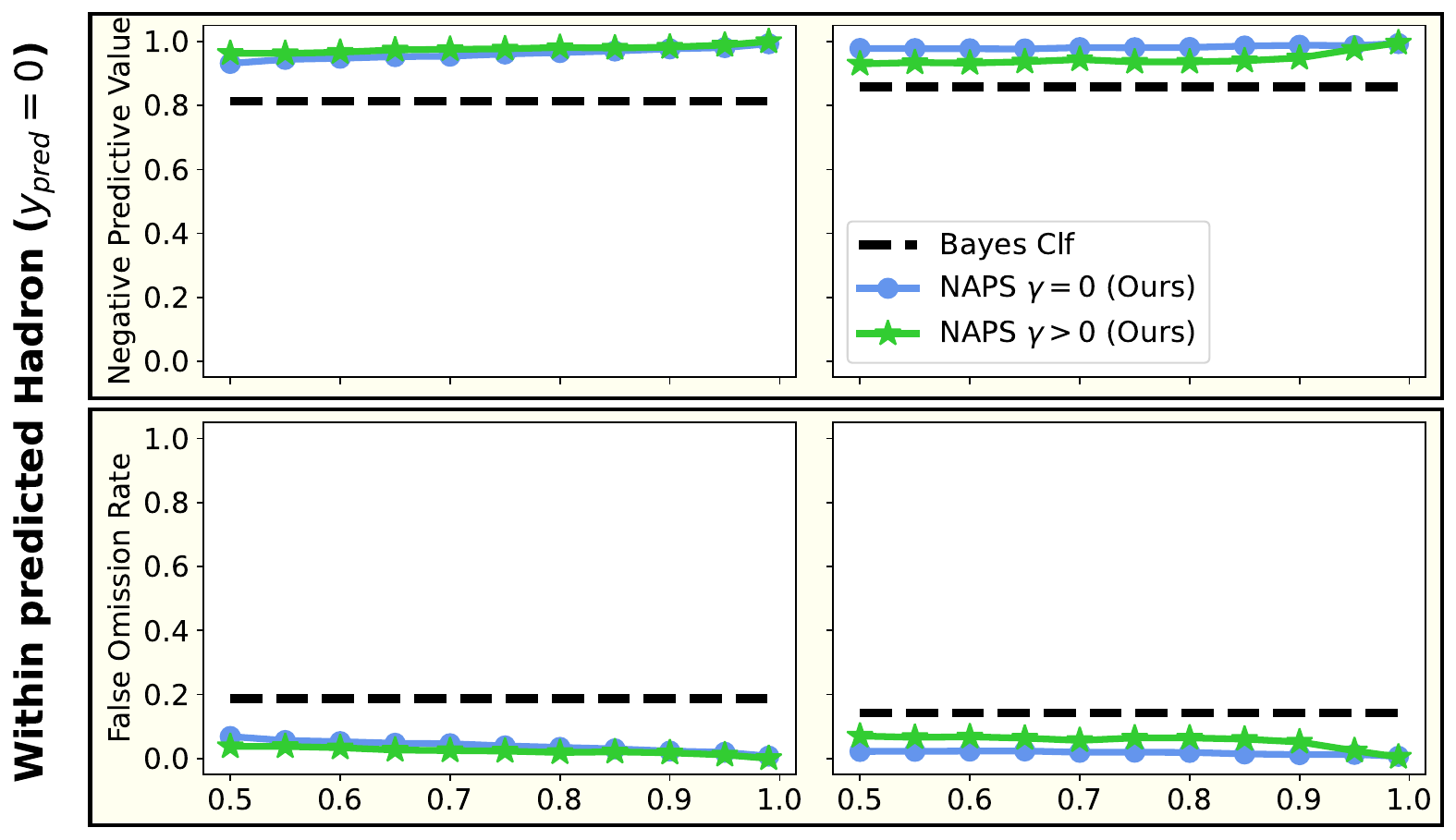}
    \caption{\textbf{Classification metrics within predicted Hadrons ($y_{\text{pred}}=0$).} Results are binned according to whether the shower energy is below (left) or above (right) the median value. Nuisance-aware prediction sets (NAPS $\gamma=0$; \textcolor{NavyBlue}{blue}) achieve high precision and low false discovery rates (FDR), especially at high confidence levels. In addition, by constraining the nuisance parameters $\bnu = (E, A, Z)$, we see performance (NAPS $\gamma>0$; \textcolor{LimeGreen}{green}) increase in the lower energy bin but with a corresponding tradeoff in the higher energy bins. Both approaches yield better results relative to the oracle Bayes classifier (black dashed line). }
    \label{fig:cr_within_ph}
\end{figure}

\begin{figure}[b!]
    \centering
    \includegraphics[width=0.68\textwidth]{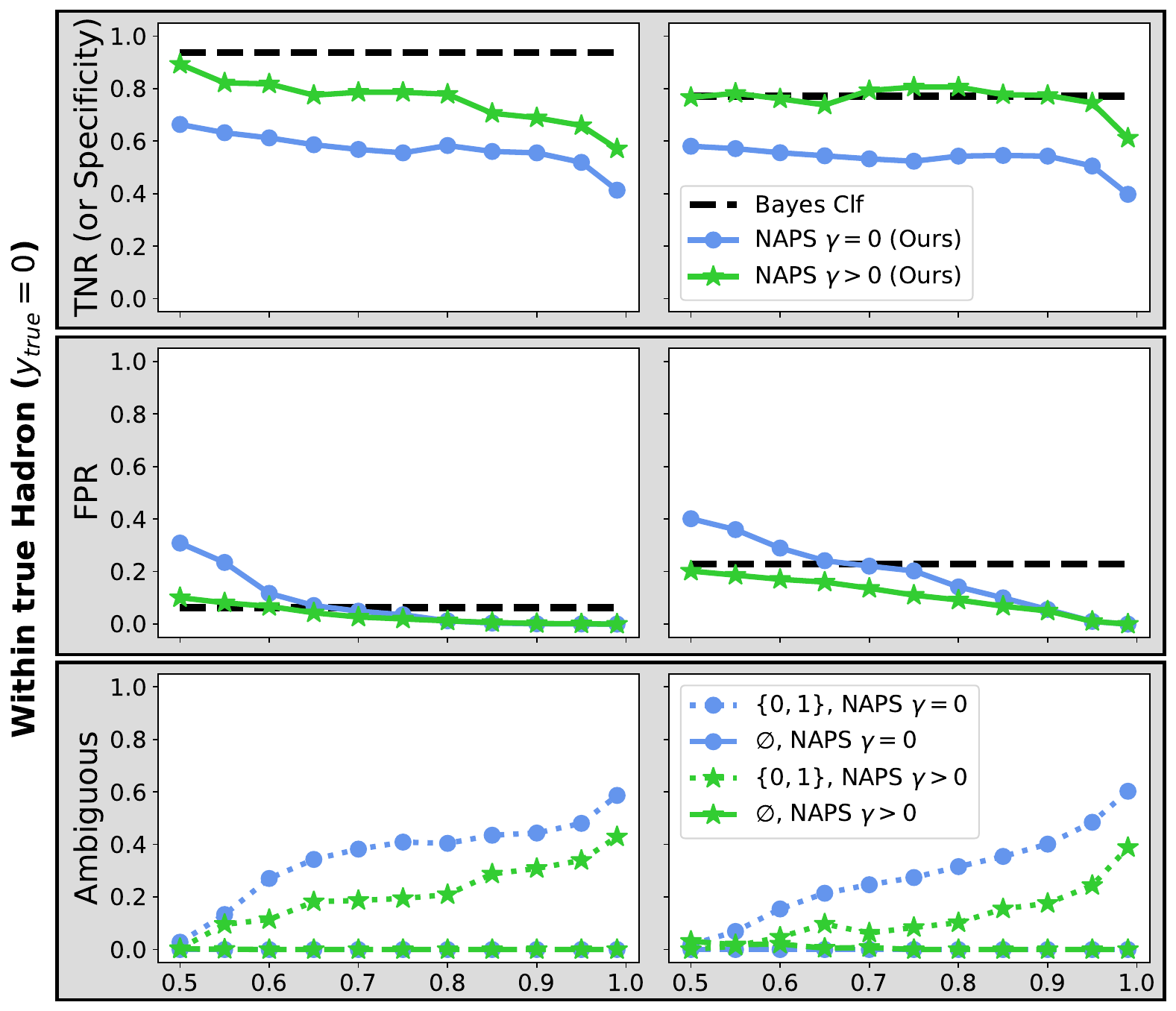}
    \caption{\textbf{Classification metrics within true Hadrons ($y=0$)}. Results are binned according to whether the shower energy is below (left) or above (right) the median value. Our set-valued classifier makes explicit its level of uncertainty on the label $y$ by returning ambiguous prediction sets (bottom row) for hard-to-classify $\x_{\target}$. Even so, NAPS with $\gamma>0$ is able to achieve a comparable number true negatives in the higher energy bins and lower number of false positives in both energy bins relative to the Bayes classifier. Here $\gamma = \alpha \times 0.3$}
    \label{fig:cr_within_h}
\end{figure}

\section{Additional Results and Details on the RNA Sequencing experiment}\label{app:rna_seq}

\subsection{Data Simulation Procedure}
The \texttt{scDesign3} simulator for RNA-Seq constructs a new simulated dataset through the following steps
\begin{enumerate}
    \item The user chooses a model type (e.g. linear with Gaussian noise) and specification to model the relationship between cell gene counts and cell features.
    \item \texttt{scDesign3} estimates model parameters on the reference data.
    \item The user supplies a matrix of all features of all cells in for the new simulated data.
    \item \texttt{scDesign3} outputs the gene counts for these cells by sampling from the estimated model.
\end{enumerate}

In our paper, we use a negative binomial GLM with cell type and batch protocol indicator as the only features: 
\[\log\E[X_{i,j} \mid Y_{j}, B_{j}] = \alpha_i + \beta_i Y_{j} + \mathbf{\gamma_i} \mathbf{B_{j}},\]
where
\begin{enumerate}
    \item $X_{i,j}$ are the observed counts for gene $i$ for cell $j$
    \item $Y_{j} \in \{0, 1\}$ is the cell type for cell $j$, \texttt{CD4\textsuperscript{+}} T-cells (Y = 1) or \texttt{Cytotoxic} T-cells (Y = 0) 
    \item $\mathbf{B}_j$ is which of the 4 protocols was used to process cell $j$, with a separate model coefficient for each protocol excluding the baseline (represented by the vector $\mathbf{\gamma}_i \in \mathbb{R}^3$)
\end{enumerate}

We also restrict our analysis to 100 genes chosen randomly from the approximately 6000 genes in the reference dataset. Although each gene count receives its own set of model parameters, new gene counts are generated in a way that captures the correlation between gene counts in the reference data. See \citep{song2023scdesign3} for more details.

The reference data used in our analysis contains two experimental protocols. One is used as a baseline to derive $\hat \alpha_i$. The second is used to fit the first entry of each $\hat{\mathbf{\gamma}}_i$, denoted $\hat\gamma_{i, 1}$. The last two entries $\hat\gamma_{i, 2}$ and $\hat\gamma_{i, 3}$ are constructed in this way:

\begin{enumerate}
    \item Each $\hat\gamma_{i, 2}$ is sampled with replacement from $\{\hat\gamma_{i, 1} : |\hat\gamma_{i, 1}| < \text{median}(\{|\hat\gamma_{j, 1}|, j \in [100]\})\}$
    \item Each $\hat\gamma_{i, 3}$ is sampled with replacement from $\{\hat\gamma_{i, 1} : |\hat\gamma_{i, 1}| \geq \text{median}(\{|\hat\gamma_{j, 1}|, j \in [100]\})\}$
\end{enumerate}

These last two batch protocols are meant to emulate a weak and stronger batch effect respectively than the different between the two original experimental protocols, while keeping realistic estimates for the effects on gene counts.

\subsection{Details on the algorithms used in Section~\ref{sec:rna_seq}}

We used gradient boosting probabilistic classifiers as implemented in \texttt{CatBoost} \citep{prokhorenkova2018catboost} to estimate both $\mathbb{P}(Y|\X)$ and $W_\lambda(C;y, \bnu)$. For the latter, \texttt{CatBoost} allows to easily enforce monotonicity constraints on the features, which we used on $C$. To compute cutoffs, we used the \texttt{brentq} routine \citep{brent2013algorithms} to calculate the inverse and the differential evolution global optimization algorithm \citep{storn1997differential} to find the infimum. Both are implemented in \texttt{SciPy} \citep{virtanen2020scipy}. The three baselines against which we compare NAPS were computed from the same base probabilistic classifier (also used for NAPS). After training it, we calibrated it on the same set used for NAPS via isotonic regression, but only for the baselines (our method has a separate calibration procedure as described in Section~\ref{sec:method}. Then we computed cutoffs as described in \citet{sadinle2019least, romano2020classification}.

\subsection{Additional Results}
Taking \texttt{CD4\textsuperscript{+}} T-cells (Y = 1) to be the positive class, Figures \ref*{fig:rna_wpp}, \ref*{fig:rna_wtp}, \ref*{fig:rna_wpn}, \ref*{fig:rna_wtn} show various performance metrics for four prediction set methodologies: standard prediction sets \citep[Theorem 1]{sadinle2019least}, class-conditional prediction sets \citep{sadinle2019least}, conformal adaptive prediction sets (APS; \citet{romano2020classification}), and NAPS with $\gamma = 0$. For many of the metrics like precision and NPV, each method achieves very good performance (perhaps due to the ease of the underlying inference problem). For TPR, we see that each method has differing strength for each of the protocols. We also notice that at very high levels of confidence, conformal APS starts outputting $\{0, 1\}$ for every observation, leading to a sharp drop in performance across all metrics.

\begin{figure}[t!]
    \centering
    \includegraphics[width=0.8\textwidth]{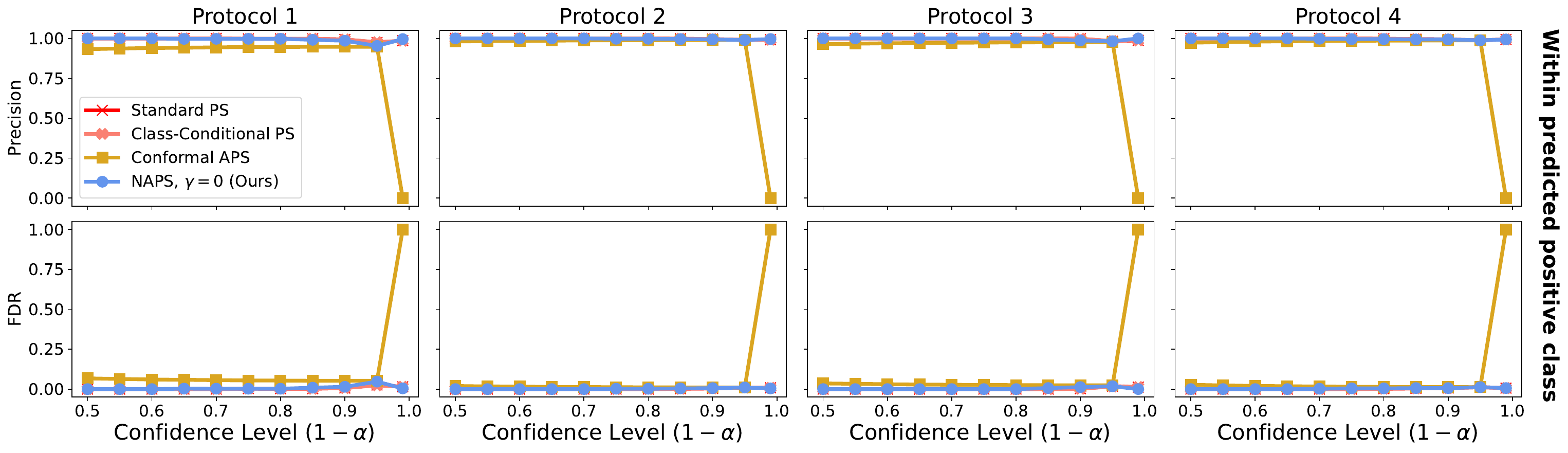}
    \caption{\textbf{Classification metrics within predicted positive class}: Precision (top) and FDR (bottom) for observations predicted to be \texttt{CD4\textsuperscript{+}} T-cells (i.e. prediction set output is $\{1\}$), additionally separated by protocol (columns). Metrics are shown for nuisance-aware prediction sets (NAPS $\gamma=0$; \textcolor{NavyBlue}{blue}), standard prediction sets (\textcolor{BrickRed}{red}), class-conditional prediction sets (\textcolor{Melon}{pink}), and conformal adaptive prediction sets (APS) (\textcolor{YellowOrange}{gold}). At high levels of confidence, conformal APS outputs $\{0, 1\}$ for all points in the test set; the corresponding metrics that require the prediction set to have one element have been set to their worst-case value.}
    \label{fig:rna_wpp}
\end{figure}

\begin{figure}[t!]
    \centering
    \includegraphics[width=0.8\textwidth]{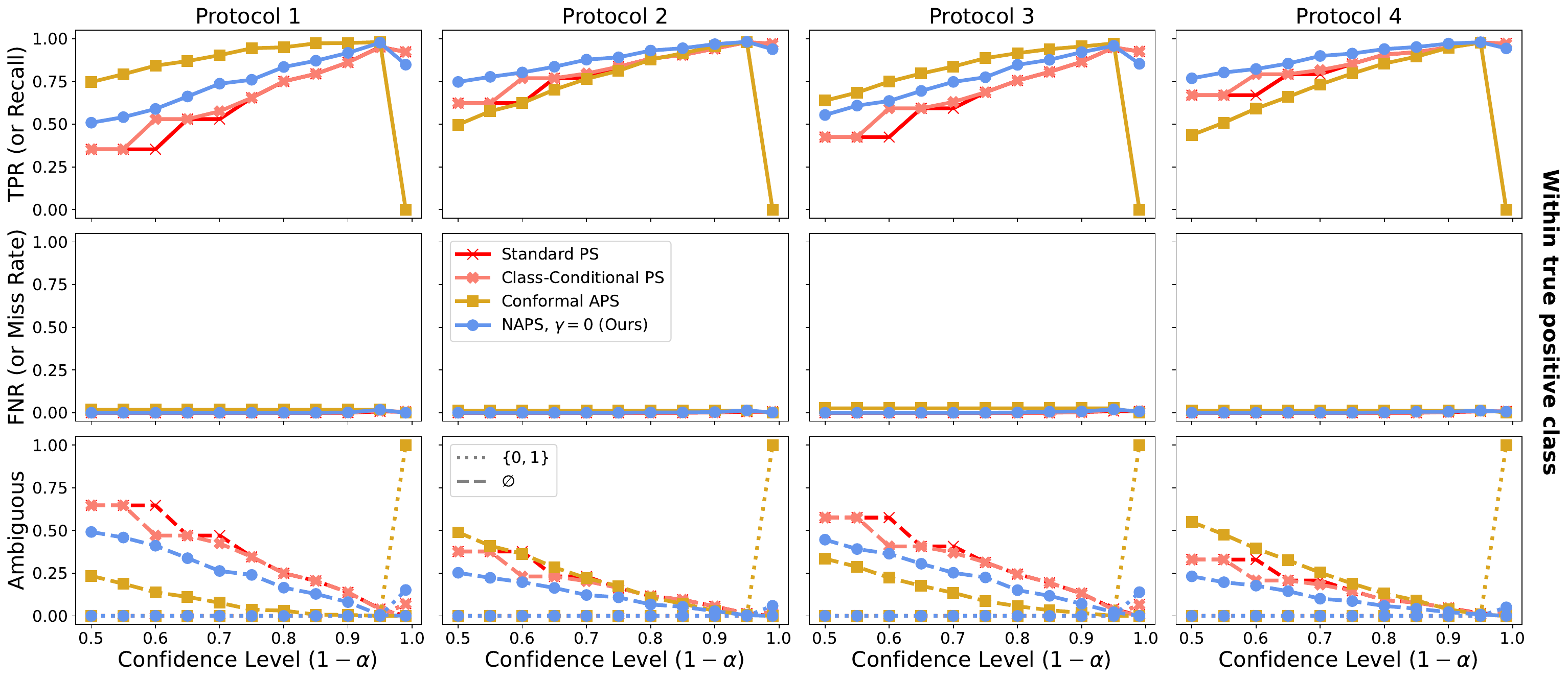}
    \caption{\textbf{Classification metrics within true positive class}: TPR (top), FNR (middle) and proportion of ambiguous sets (bottom) for true \texttt{CD4\textsuperscript{+}} T-cells, additionally separated by protocol (columns). Metrics are shown for Nuisance-aware prediction sets (NAPS $\gamma=0$; \textcolor{NavyBlue}{blue}), standard prediction sets (\textcolor{BrickRed}{red}), class-conditional prediction sets (\textcolor{Melon}{pink}), and conformal adaptive prediction sets (APS) (\textcolor{YellowOrange}{gold}). At high levels of confidence, conformal APS outputs $\{0, 1\}$ for all points in the test set; the corresponding metrics that require the prediction set to have one element have been set to their worst-case value.}
    \label{fig:rna_wtp}
\end{figure}

\begin{figure}[h!]
    \centering
    \includegraphics[width=0.8\textwidth]{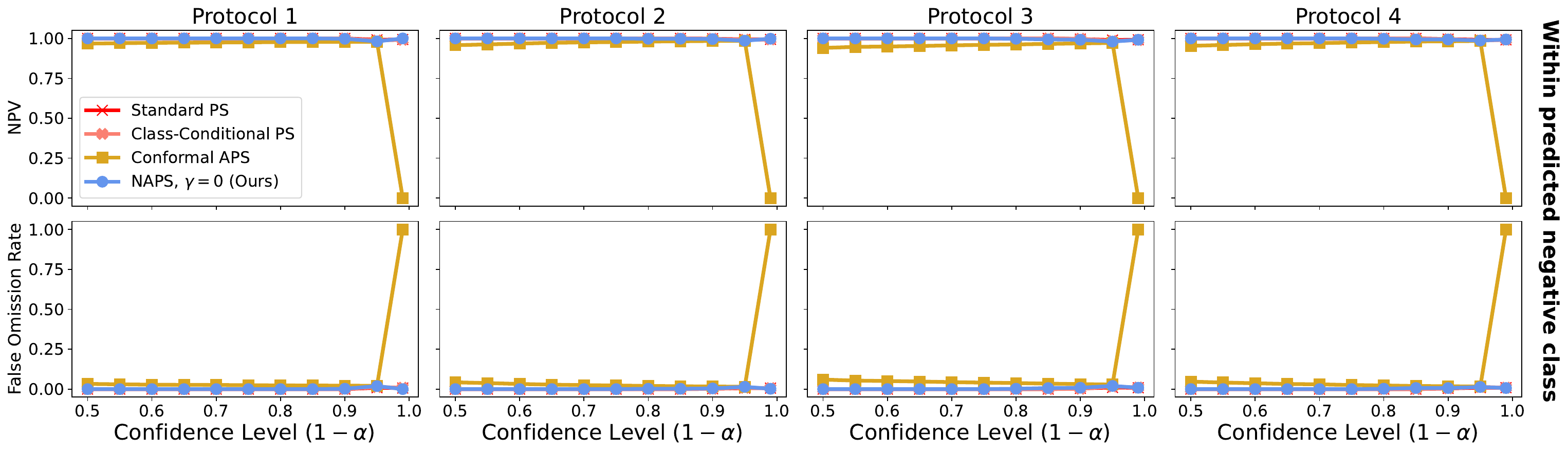}
    \caption{\textbf{Classification metrics within predicted negative class}: NPV (top) and False Omission Rate (bottom) for observations predicted to be \texttt{Cytotoxic} T-cells (i.e. prediction set output is $\{0\}$), additionally separated by protocol (columns). Metrics are shown for Nuisance-aware prediction sets (NAPS $\gamma=0$; \textcolor{NavyBlue}{blue}), standard prediction sets (\textcolor{BrickRed}{red}), class-conditional prediction sets (\textcolor{Melon}{pink}), and conformal adaptive prediction sets (APS) (\textcolor{YellowOrange}{gold}). At high levels of confidence, conformal APS outputs $\{0, 1\}$ for all points in the test set; the corresponding metrics that require the prediction set to have one element have been set to their worst-case value.}
    \label{fig:rna_wpn}
\end{figure}

\begin{figure}[t!]
    \centering
    \includegraphics[width=0.8\textwidth]{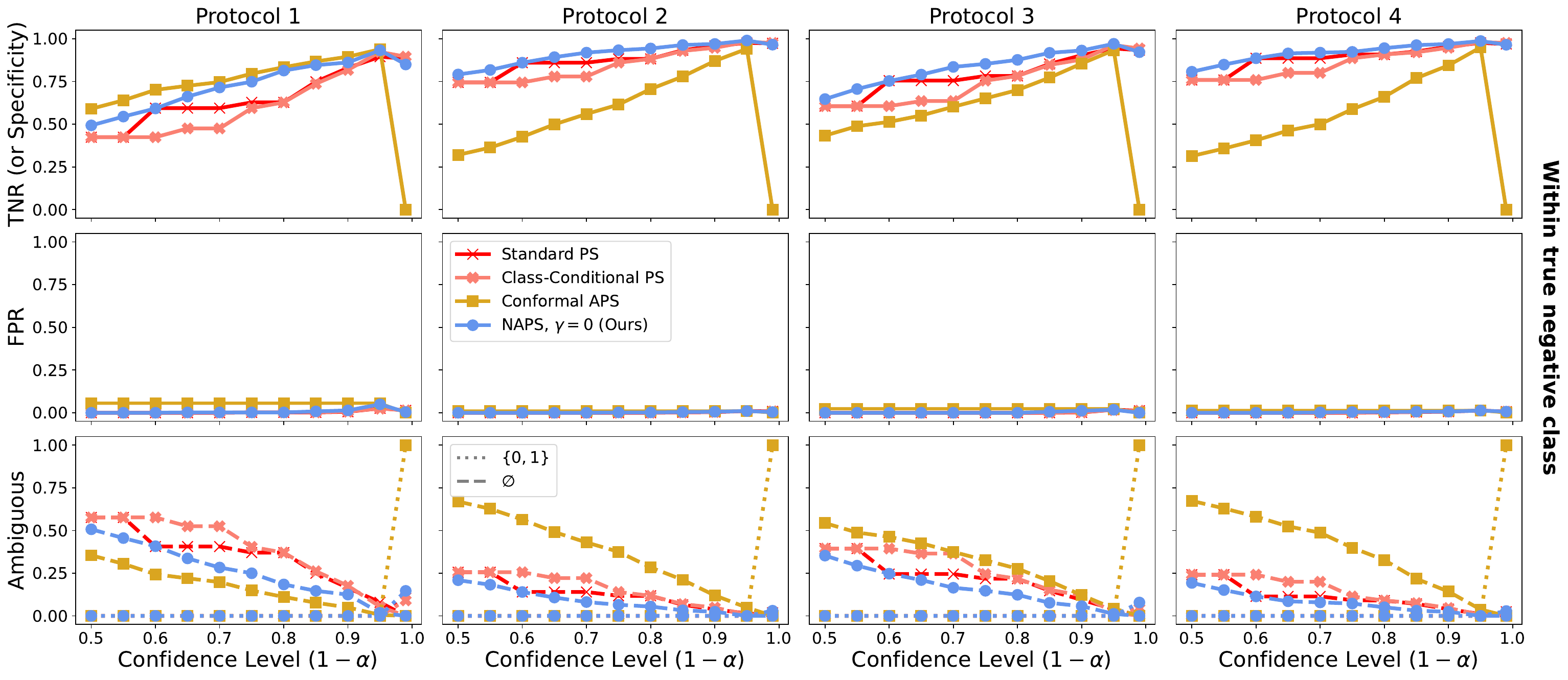}
    \caption{\textbf{Classification metrics within true negative class}: TNR (top), FPR (middle) and proportion of ambiguous sets (bottom) for true \texttt{Cytotoxic} T-cells, additionally separated by protocol (columns). Metrics are shown for Nuisance-aware prediction sets (NAPS $\gamma=0$; \textcolor{NavyBlue}{blue}), standard prediction sets (\textcolor{BrickRed}{red}), class-conditional prediction sets (\textcolor{Melon}{pink}), and conformal adaptive prediction sets (APS) (\textcolor{YellowOrange}{gold}). At high levels of confidence, conformal APS outputs $\{0, 1\}$ for all points in the test set; the corresponding metrics that require the prediction set to have one element have been set to their worst-case value.}
    \label{fig:rna_wtn}
\end{figure}

\section{\add{Computational Analysis: Training and Inference Times}}
\label{sec:comp_times}

\add{Table~\ref{tab:comp_times} reports training and inference times for NAPS under the Single-Cell RNA Sequencing (Section~\ref{sec:rna_seq}) and Atmospheric Cosmic-Ray Showers (Section~\ref{sec:cosmic_rays}) experiments. Dataset sizes are the proportions included in the training, calibration and inference sets out of the total number of simulations indicated in Sections~\ref{sec:rna_seq} and Section~\ref{sec:cosmic_rays}. For calibration, we report the time needed to estimate ROC curves from the augmented calibration set, including “re-calibration” of the estimated rejection probabilities via isotonic regression. For NAPS with $\gamma > 0$ (only performed in Section 5.3), inference times are measured per-observation (on average) since cutoffs are data-dependent and need to be computed for each $\mathbf{x}$. For NAPS with $\gamma = 0$, we report the total time needed to compute cutoffs, as they can then be applied to any new observation $\mathbf{x}$ (i.e., they are amortized with respect to observations). Once this is done, constructing the prediction sets takes only a few milliseconds. All times are computed for inference at a single level $\alpha$. Classifier training and the calibration procedure only need to be estimated once (here we report times that include five-fold cross-validation). All computations were performed on a MacBook Pro M1Pro with 16 GB of RAM.}

\begin{table}[b!]
\caption{\add{Training and inference times for NAPS for the experiments of Sections~\ref{sec:rna_seq} and ~\ref{sec:cosmic_rays}.}}
\label{tab:comp_times}
\vskip 0.15in
\begin{center}
\begin{small}
\begin{sc}
\begin{tabular}{lcccccr}
\toprule
\textbf{Experiment} & \textbf{Dataset Size} & \textbf{Training} & \textbf{Calibration} & \textbf{Inference} ($\gamma = 0$) & \textbf{Inference} ($\gamma > 0$) \\
\midrule
RNA-Seq & $0.6, 0.35, 0.5$ & 6 minutes & 30 minutes & 1 second & / \\
Cosmic Rays & $0.45, 0.45, 0.1$ & 8 minutes & 65 minutes & 6 seconds & 4 seconds per-obs\\
\bottomrule
\end{tabular}
\end{sc}
\end{small}
\end{center}
\vskip -0.1in
\end{table}

\section{Synthetic Example: Deep Dive} 
\label{sec:deep-dive}

\subsection{Impact of the Nuisance Parameter}
As mentioned in the main text, we consider a process that generates events $(Y_i, X_i)$, where $Y_i \in \{0, 1\}$ determines the type or label of the event, and $X_i \in [0, 1]$ is the sole feature of the event. The distribution of events is defined as follows

\begin{enumerate}
    \item $\P[Y_i = 0] = \P[Y_i = 1] = 1/2$
    \item Conditional density for $Y = 1$: $\displaystyle p(\x_i \mid Y_i = 1) = \frac{e^{\x_i}}{e-1}$
    \item Conditional density for $Y = 0$: $\displaystyle p(\x_i \mid Y_i = 0, \nu_i) = \frac{\nu_i e^{-\nu_i \x_i}}{1-e^{-\nu_i}}$
\end{enumerate}

Where $\nu$ is an additional nuisance parameter that influences the density of $X$ for $Y = 0$ events. $\nu_i$ is assumed to be drawn from some distribution independently for each $Y = 0$ event. We are interested in inferring $Y$ given observed $\X$ and unobserved $\nu$. Figure \ref{fig:synth_overview} shows how the presence of the nuisance parameter affects this inference task.

\begin{figure}[t!]
    \centering
    \includegraphics[width=0.8\textwidth]{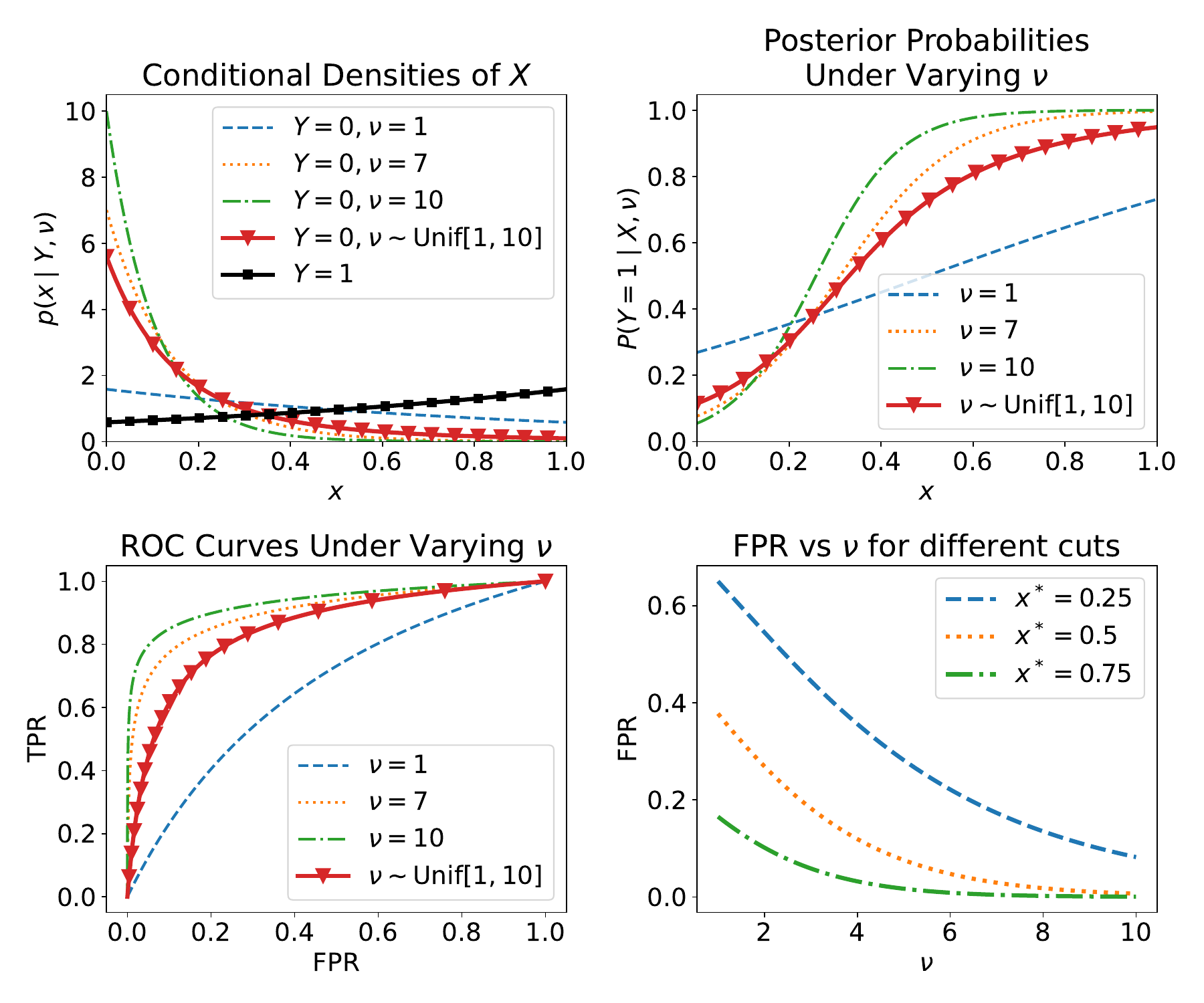}
    \caption{\textbf{Impacts of Nuisance Parameters on the Inference Task} \emph{Top Left}: Conditional densities $p(x \mid Y, \nu)$ for various values of $Y$ and $\nu$ according to the problem setup. The marginal density $p(x \mid Y = 0)$  shown in red is induced by a Unif$[1, 10]$ prior on $\nu$. \emph{Top Right}: Posterior probability $P(Y = 1 \mid X, \nu)$ as a function of $X$ for different values of the nuisance parameter $\nu$. The marginal posterior $P(Y = 1 \mid X)$ is shown in red for a Unif$[1, 10]$ prior on $\nu$. \emph{Bottom Left}: ROC curves for the Bayes Classifier holding $\nu$ fixed (blue, orange, and green curves) and for a Unif$[1, 10]$ prior on $\nu$ (red). $Y = 1$ is taken to be the positive class. \emph{Bottom Right}: Under the classification rule that $\hat y_i = 1$ if $x_i > x^*$, this figure shows how the FPR of that classifier will vary with $\nu$. Each curve represents a different cut $x^*$ for the classification rule. }
    \label{fig:synth_overview}
\end{figure}

The top left of Figure \ref{fig:synth_overview} demonstrates how the shape of the density of $X$ for $Y = 0$ events can vary dramatically depending on the value of $\nu$. Assuming any prior of $\nu$ can yield a density of $X$ that does not depend on $\nu$, but it may not closely resemble the conditional densities of $X$ given $\nu$ for all values of $\nu$. The top right panel shows how this variation in the shape of the densities subsequently affects the behavior of the posterior probabilities of $Y$ given $X$ and $\nu$. Again, we can derive a posterior that does not depend on $\nu$, with the same caveat as before. We also observe that the posterior probabilities are always monotonic in $x$, therefore any classifier or prediction set that uses cutoffs on posterior probabilities can be equivalently defined using cutoffs on $x$ directly. The bottom left figure shows how the ROC for the Bayes Classifier (i.e. directly using the posterior probabilities to classify events) can vary under fixed $\nu$ or a prior on $\nu$. These ROC curves demonstrate why ignoring nuisance parameters can yield biased or otherwise unreliable results. Every fixed value of $\nu$ as well as every prior on $\nu$ yields a completely different relationship between FPR and TPR. The bottom right figure shows that if our goal is valid FPR control for our inference task, we must take the nuisance parameter into account. Because the ultimate FPR for any cutoff depends on the value of $\nu$ for each observation, the selection of an cutoff that controls FPR must properly account for the influence of the nuisance parameter.

\subsection{Additional Results}

Figures \ref{fig:synth_covdeep}, \ref{fig:synth_powerdeep}, and \ref{fig:synth_precisiondeep} show additional results from the synthetic examples for both standard prediction sets and class-specific prediction sets used in the cosmic ray application. All prediction sets are formed under the training prior $\nu \sim \text{Unif}[1, 10]$, which is the same prior used to compute metrics under the ``No GLS'' setting. ``With GLS'' changes the target prior to $\nu \sim \mathcal{N}(4, 0.1)$ without modifying the training prior. Coverage for $Y = 1$ events, power for $Y = 0$ events (defined as $\P[1 \notin \text{Prediction Set} \mid Y = 0]$), and precision for $\{0\}$ outputs do not vary significantly across methodologies due to the fact that $p(\x_i \mid Y_i = 1)$ does not depend on the distribution of $\nu_i$. As seen in the text, our methods achieve validity regardless of the presence of GLS. We also achieve higher precision than standard or class-specific prediction sets, although we do sacrifice power compared to those methods. However, careful selection of $\gamma$ in the NAPS framework can help increase power without losing validity.

\begin{figure}[t!]
    \centering
    \includegraphics[width=0.8\textwidth]{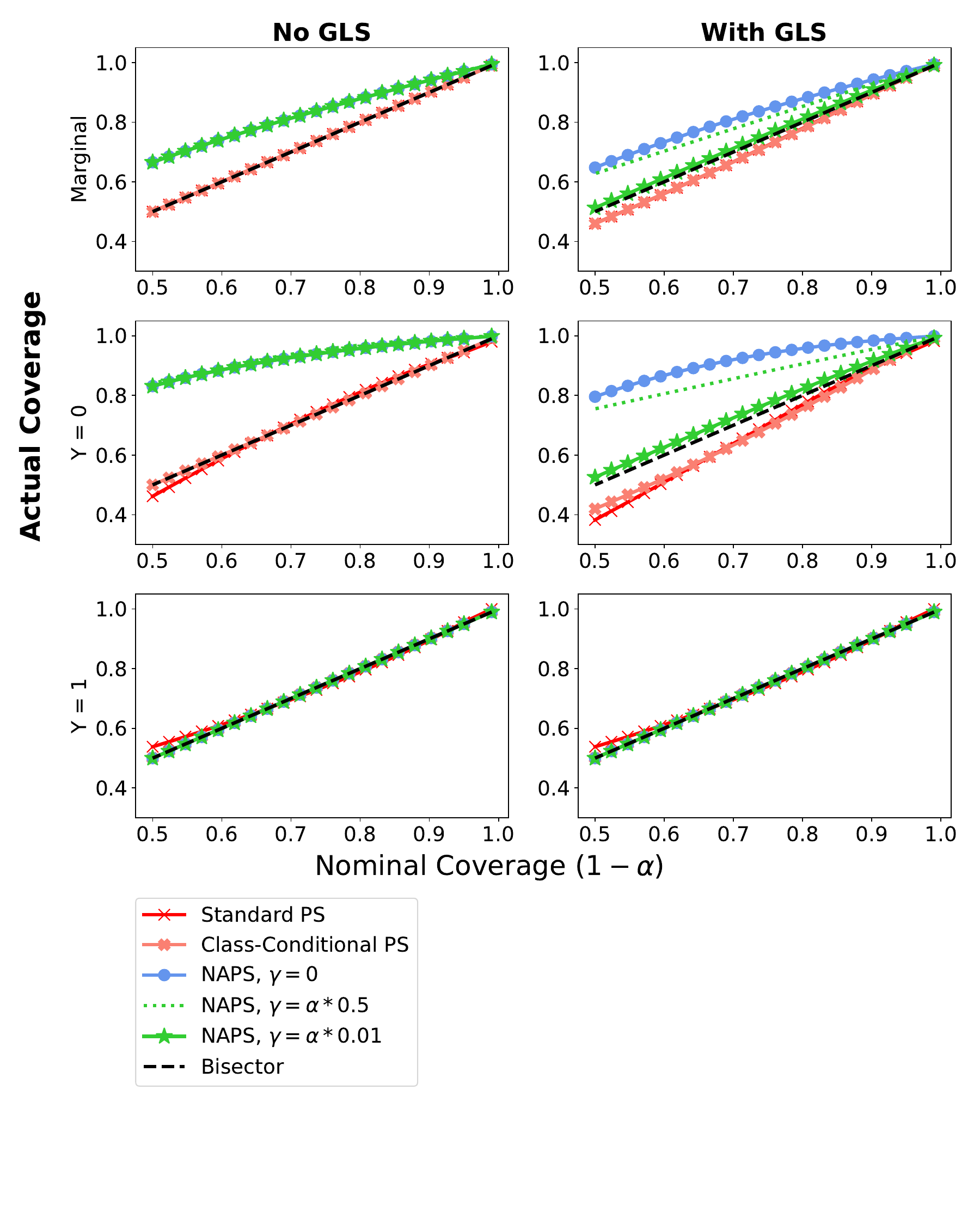}
    \caption{\textbf{Actual vs Nominal Coverage for Several Prediction Set Methods}: We compare the actual coverage of standard prediction sets (red), class-specific prediction sets (pink), and NAPS under different $\gamma$ values under no GLS (left) and with GLS (right). We show marginal coverage (top), and conditional coverage for $Y = 0$ events (middle) and $Y = 1$ events (bottom)}
    \label{fig:synth_covdeep}
\end{figure}

\begin{figure}[t!]
    \centering
    \includegraphics[width=0.8\textwidth]{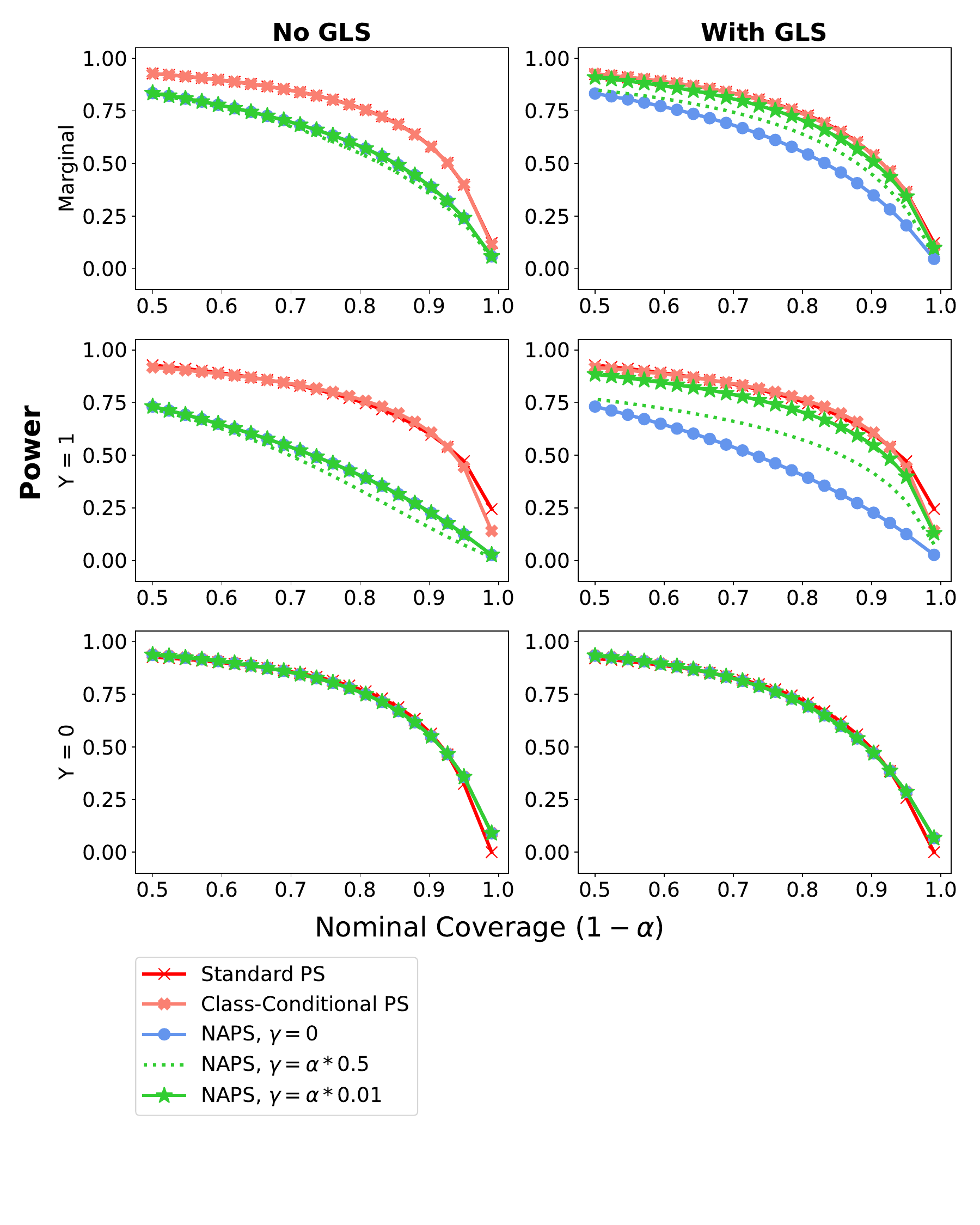}
    \caption{\textbf{Power vs Nominal Coverage for Several Prediction Set Methods}:  We compare the power of standard prediction sets (red), class-specific prediction sets (pink), and NAPS under different $\gamma$ values under no GLS (left) and with GLS (right). Power for $Y = 0$ events (bottom) is defined as $\P[1 \notin \text{Prediction Set} \mid Y = 0]$ and vice versa for $Y = 1$ (middle). Marginal power (top) is the sum of these two power metrics weighted by $\P[Y = 1]$. }
    \label{fig:synth_powerdeep}
\end{figure}

\begin{figure}[t!]
    \centering
    \includegraphics[width=0.8\textwidth]{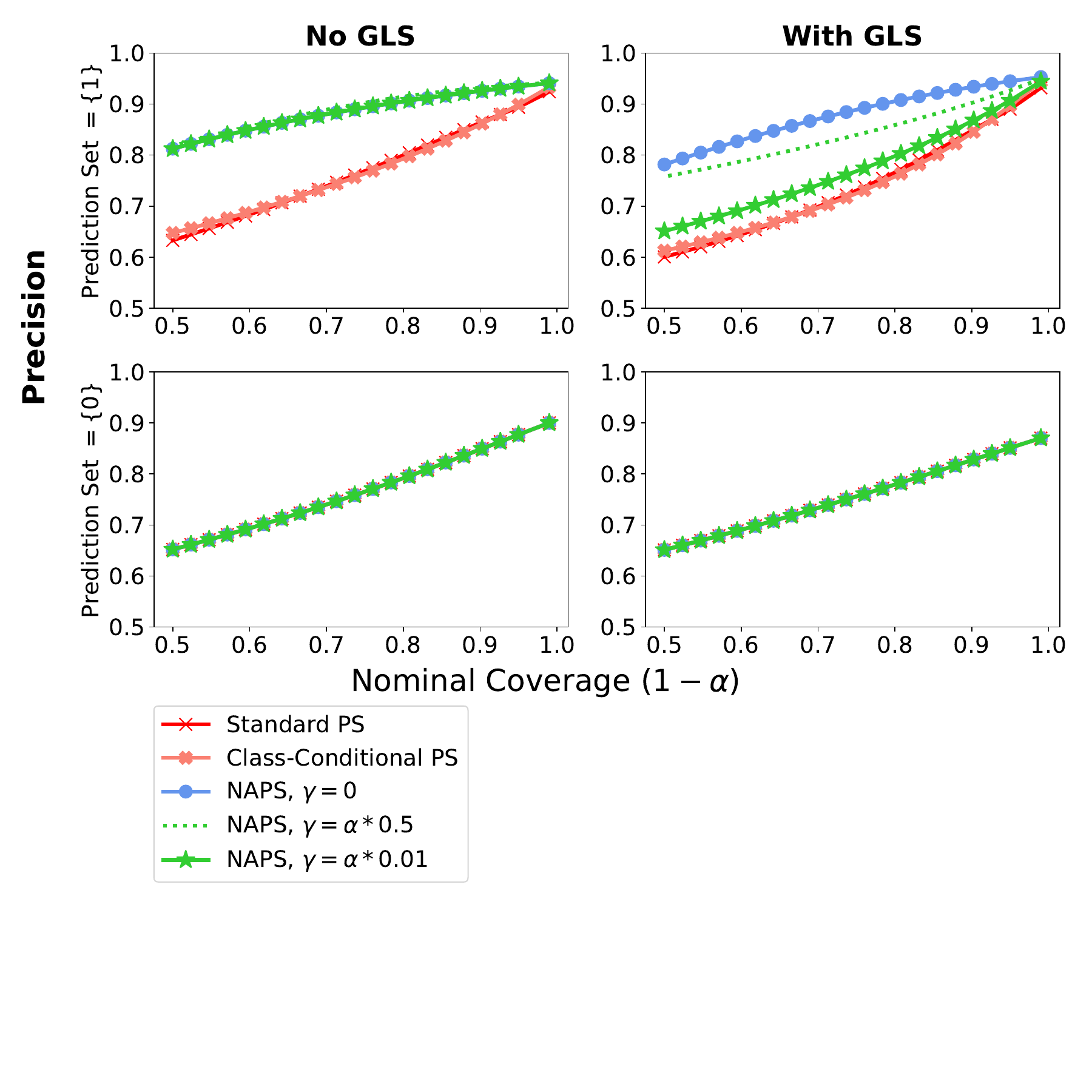}
    \caption{\textbf{Precision vs Nominal Coverage for Several Prediction Set Methods}:  We compare the precision of standard prediction sets (red), class-specific prediction sets (pink), and NAPS under different $\gamma$ values under no GLS (left) and with GLS (right). We define precision for $\text{prediction set} = \{0\}$ as $\P[Y = 0 \mid \text{prediction set} = \{0\}]$ and vice versa for $\text{prediction set} = \{1\}$ outputs. Events where $\text{prediction set} = \{0, 1\}$ or $\text{prediction set} = \emptyset$ are not considered here.}
    \label{fig:synth_precisiondeep}
\end{figure}

\subsection{$\nu$-Conditional Coverage and validity under GLS}

Figure \ref{fig:synth_covextra} below explores coverage of different prediction set methods conditional on $Y$ \emph{and} $\nu$, under the training prior $\nu \sim \text{Unif}[0, 1]$. We compare 4 methods:

\begin{enumerate}
    \item Standard prediction sets that target marginal coverage only
    \item Class-conditional prediction sets that target coverage conditional on $Y$
    \item Class-conditional prediction sets that additionally use the posterior mean  $\hat \nu(x) = \int_\mathcal{N} \nu \: p(\nu \mid x) d\nu$ as an point estimate of $\nu$ to evaluate the posterior. Specifically, $P[Y = 1 \mid X, \nu = \hat\nu(X)]$ is used instead of $P[Y = 1 \mid X]$, where the latter integrates over the prior on $\nu$
    \item NAPS with $\gamma = 0$
\end{enumerate}

Method 3 is added as a possible alternative to forming confidence sets on $\nu$ within the NAPS framework. The figure shows that, although standard and class-conditional prediction sets achieve marginal and class-conditional validity respectively, they do not maintain validity when conditioning on all values of $\nu$. This is the fundamental reason that these methods do not achieve validity under GLS. Whereas, NAPS achieves validity conditional on both $Y$ and $\nu$, resulting in robustness to GLS. We note that method 3 achieves neither marginal nor class-conditional validity, indicating that even well-formed point estimates of $\nu$ are insufficient to reach nominal coverage levels. 

\begin{figure}[t!]
    \centering
    \includegraphics[width=\textwidth]{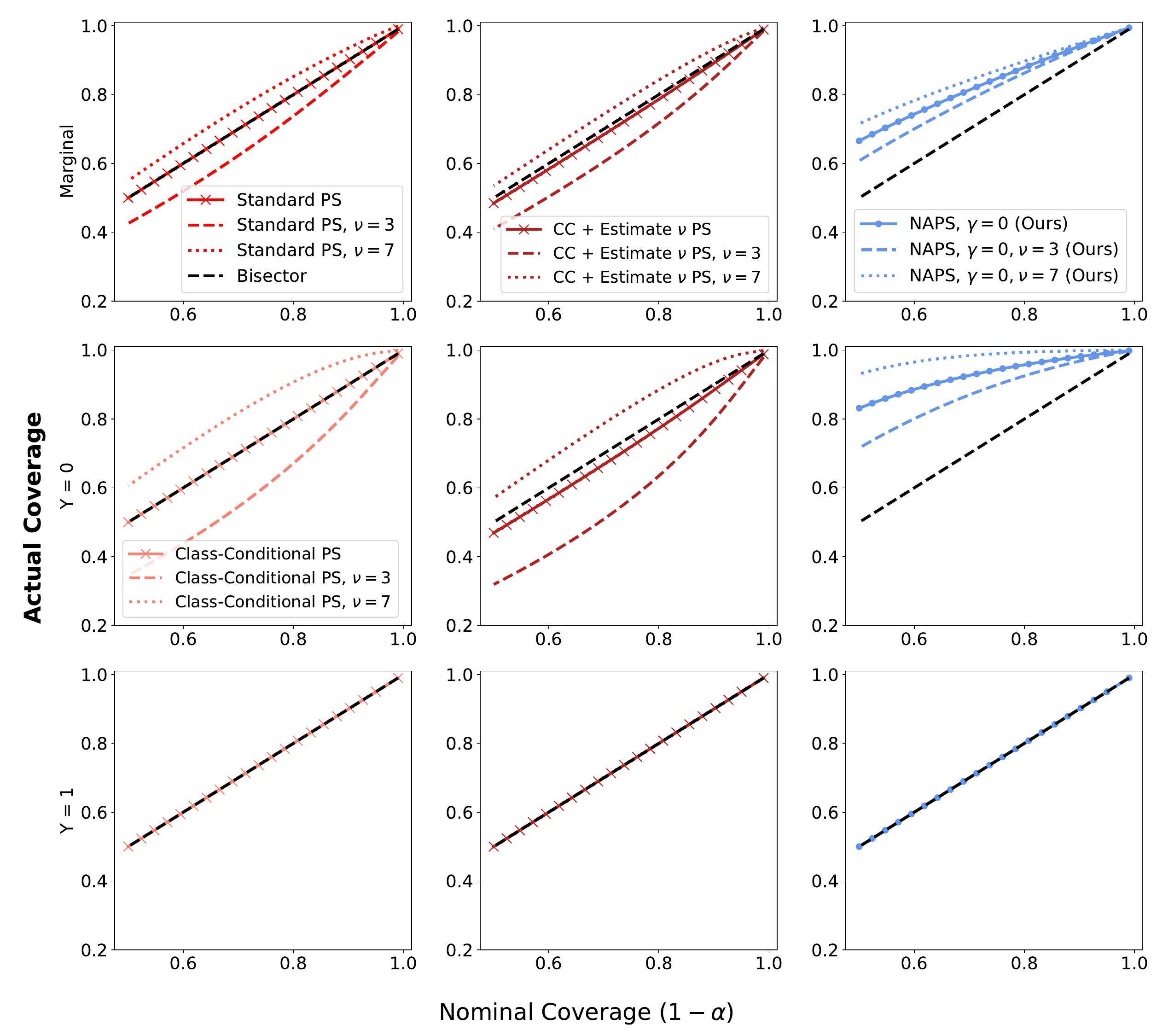}
    \caption{\textbf{Marginal, Class-conditional, and $\nu$-conditional Coverage of Several Prediction Set Methods}: We examine marginal coverage under the training prior on $\nu$ (top), $Y=0$ conditional coverage (middle) and $Y = 1$ conditional coverage (bottom) for standard prediction sets (red, top right only), class-conditional prediction sets (pink, middle left and bottom left), class-conditional prediction sets with estimated $\nu$ (dark red, middle column), and NAPS (blue, right column). In each figure, we also show coverage when additionally conditioning on certain values of $\nu$ (dotted and dashed lines)}
    \label{fig:synth_covextra}
\end{figure}

\subsection{When does $\gamma > 0$ for NAPS increase power?}
\label{sec:increase_power}

The $\gamma$ parameter for NAPS gives us the option to first form a confidence set for $\nu$ on a new observation $x$ before optimizing the cutoffs for our test statistic (see Section \ref{sec:theory}). Because the test statistic is monotonic in the posterior probabilities, we can derive cutoffs on $x$ directly based on the confidence set for $\nu$. Specifically, we can simplify the procedure in Theorem \ref{lemma:NA_cutoff} to the following
\[x_0(\nu; \alpha, \gamma) = x \hspace{4mm} \text{s.t.} \hspace{4mm} \P[X \geq x \mid Y = 0, \nu] = \alpha - \gamma\]

\[x_0^*(\alpha) = \sup_{\nu \in S_0(x; \gamma)}x_0(\nu, \alpha)\]

\[x_1^*(\alpha) = x \hspace{4mm} \text{s.t.} \hspace{4mm} \P[X \leq x \mid Y = 1] = \alpha\]

Where $S_0(x; \gamma)$ is a $1 - \gamma$ confidence set on $\nu$ given $Y = 0$. Then, our prediction set becomes 

\[0 \in \H(x; \alpha) \text{ if } x < x_0^*(\alpha)\]
\[1 \in \H(x; \alpha) \text{ if } x > x_1^*(\alpha)\]

We note that $x_1^*(\alpha)$ does not depend on our choice of $\gamma$, so we focus on $x_0^*(\alpha)$. We also note that lower values of $x_0^*(\alpha)$ result in higher power of the final NAPS. Figure \ref*{fig:synth_nucs} below shows how the choice of $\gamma$ can affect the power of the resulting NAPS.

\begin{figure}[t!]
    \centering
    \includegraphics[width=\textwidth]{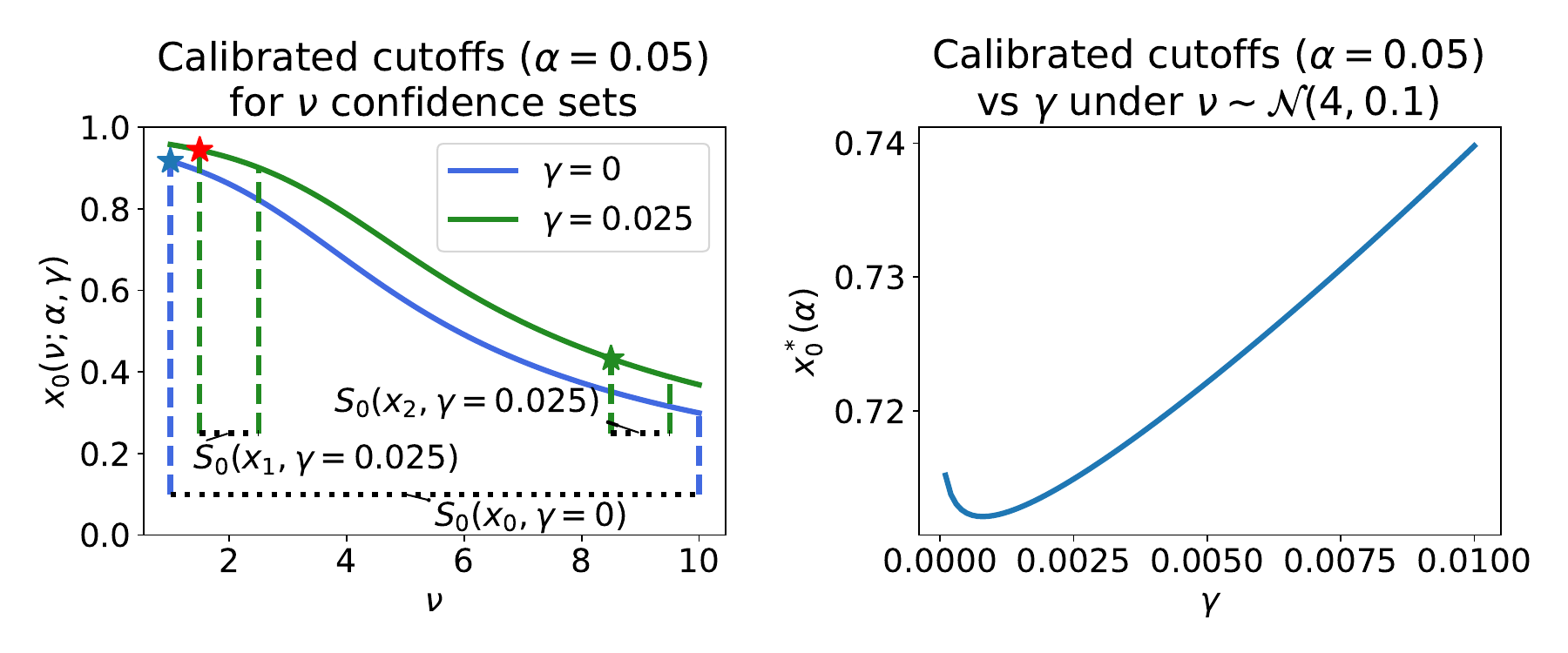}
    \caption{\textbf{Effect of $\gamma$ on NAPS Power} \emph{Left}: We show how the optimization of $x_0(\nu; \alpha, \gamma)$ depends on $\gamma$ and $S_0(x; \gamma)$. The two curves show the relationship between $x_0(\nu; \alpha, \gamma)$ and $\nu$ under two values of $\gamma$. When $\gamma = 0$, we must optimize over the entire space of $\nu$ to derive $x_0^*(\alpha)$ (or equivalently,  $S_0(x; \gamma = 0) = [1, 10]$ for all $x$. This leads to a $x_0^*(\alpha)$ value indicated by the blue star. When $\gamma = 0.0025$, we consider two hypothetical confidence sets $S_0(x_1; \gamma)$ and $S_0(x_2; \gamma)$ for $\nu$, indicated by the two pairs of green dotted lines. In each case, we only optimize $x_0(\nu; \alpha, \gamma)$ over the values of $\nu$ in the confidence set; however, to maintain coverage at $1 - \alpha$, optimization is done over the green curve instead of the blue curve. Optimization over $S_0(x_1; \gamma)$ yields $x_0^*(\alpha)$ indicated by the red star, while optimization over $S_0(x_2; \gamma)$ yields $x_0^*(\alpha)$ indicated by the green star. \emph{Right}: When $S_0(x; \gamma)$ is taken to be the $(\gamma/2, 1-\gamma/2)$ quantiles of the truncated $\mathcal{N}(4, 0.1)$ distribution for all $x$, we can derive a relationship between $x_0^*(\alpha)$ and $\gamma$. In this case, the calibrated cutoff is minimized at $\gamma \approx 0.001$. }
    \label{fig:synth_nucs}
\end{figure}

The left panel demonstrates the tradeoff inherent in selection a value of $\gamma$. Fixing $\nu$ and $\alpha$, $x_0(\nu; \alpha, \gamma)$ is increasing in $\gamma$ (illustrated by the green curve being always higher than the blue curve), so the cutoff at every $\nu$ will always be higher (and power subsequently lower). However, constraining $\nu$ to $S_0(x; \gamma)$ may avoid optimizing over regions of $\nu$ where $x_0(\nu; \alpha, \gamma)$ is relatively high (i.e. small values of $\nu$). In the synthetic example, the most power is gained when $S_0$ constrains $\nu$ to a region where $\nu$ is much larger than 1 (the value of $\nu$ that yields $x_0^*(\alpha)$ when $\gamma = 0$). This is illustrated by the fact that $S_0(x_2; \gamma = 0.0025)$ yields a $x_0^*(\alpha)$ value (green star) much lower than the value obtained when $\gamma = 0$ (blue star). However, setting $\gamma > 0$ can sometimes result in power \emph{loss} if $S_0$ contains small values of $\nu$. This is illustrated by the fact that $S_0(x_1; \gamma = 0.0025)$ yields an even higher $x_0^*(\alpha)$ value (red star) than the case when $\gamma = 0$. The right panel shows that, in our simple synthetic example, there is a relatively clear optimal value for $\gamma$ which is non-zero. 

In general, the distribution of the nuisance parameter(s) and the efficiency of the confidence sets on those NPs will determine which value of $\gamma$ is optimal. If most data points have nuisance parameter values in ``favorable'' regions of the NP space, then it may be worth setting $\gamma > 0$ to form confidence sets. In other cases, letting $\gamma = 0$ may be the optimal choice.

\add{\subsection{Performance of NAPS under SLS}

In the synthetic example, we assumed that the distribution of labels $\P[Y = 1]$ was the same for the training and target data. However, the distribution of $\nu$ is not the same, which leads to $p_{\text{train}}(x \mid Y) \neq p_{\text{target}}(x \mid Y)$, since

\[
    p(x \mid Y = y) = \int p(x \mid Y = y, \nu)\pi(\nu \mid Y = y) \: d \nu
\]

and we explicitly allow for a change in $\pi(\nu \mid Y = y)$ under GLS. This setup is essentially the reverse of the Standard Label Shift (SLS) setup. Under SLS, we would assume that $\P_{\text{train}}[Y = 1] \neq \P_{\text{train}}[Y = 1]$, but that $p_{\text{train}}(x \mid Y) = p_{\text{target}}(x \mid Y)$, which is most directly achieved when the distribution of $\nu$ does not change between the training and target data.

We have shown that class-conditional prediction sets (designed to maintain coverage under SLS) do not maintain coverage under GLS due to the violation of the assumption that $p_{\text{train}}(x \mid Y) = p_{\text{target}}(x \mid Y)$. In this section, we explore how NAPS performs in the SLS setting relative to class-conditional prediction sets. We expect NAPS coverage guarantees to hold, with a decrease in power due to NAPS enforcing nominal coverage at every point in the nuisance parameter space. Figure \ref{fig:naps_sls} shows the results of our experiments under SLS.

\begin{figure}[H]
    \centering
    \includegraphics[width=\textwidth]{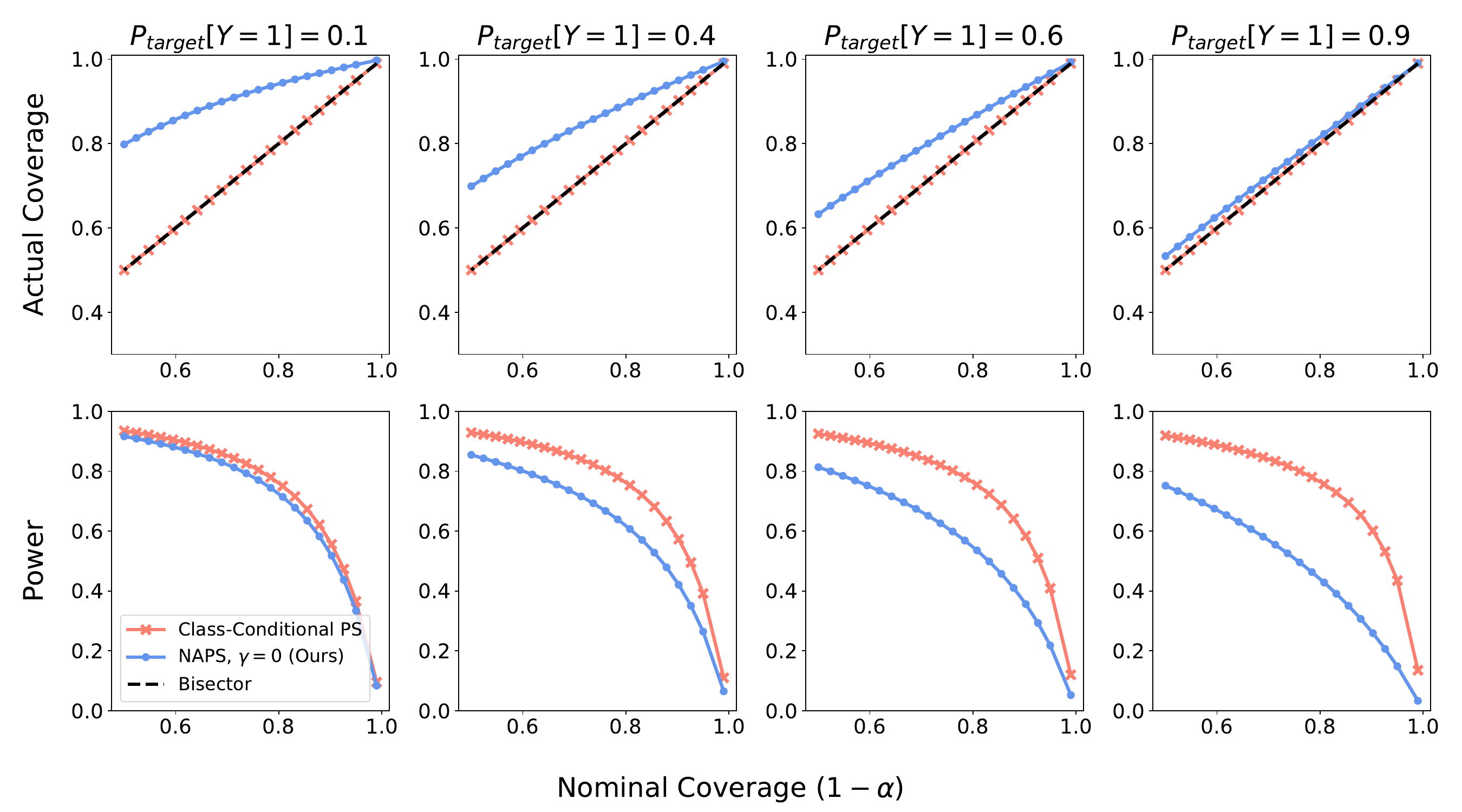}
    \caption{\textbf{Comparison of NAPS and Class-Conditional Prediction Sets under Standard Label Shift}: We plot the test set marginal coverage (top row) and marginal power (bottom row, defined as $\P_{\text{target}}[1-Y \notin \text{prediction set}]$). We compare NAPS (blue) to Class-Conditional PS (pink). This comparison is done for several levels of SLS (columns), where we shift the distribution $Y$ in the evaluation set from $\P_{\text{train}}[Y=1] = 0.5$. \add{The distribution of the nuisance parameter $\nu$ is {\em the same} for training versus target data; that is, we have an SLS setting}.}
    \label{fig:naps_sls}
\end{figure}

In all SLS scenarios we tested, NAPS over-covers and achieves lower levels of power compared to class-conditional prediction sets, demonstrating the theoretical tradeoff described above. Looking at coverage, we see that as $\P_{\text{target}}[Y=1]$ increases, the level of overcoverage for NAPS decreases. This is expected, since the nuisance parameter $\nu$ only affects the distribution of features for $Y=0$ events and causes NAPS to exclude 0 from the prediction set less often. Unsurprisingly, class-conditional prediction sets exactly achieve nominal coverage under every SLS scenario. 

Looking at power, we note that class-conditional prediction sets achieve similar (but not identical) power across all SLS scenarios. Power for NAPS appears to decrease as $\P_{\text{target}}[Y=1]$ increases. This is a consequence of the same fact that $\nu$ only affects $Y=0$ events; because NAPS will exclude 0 from its prediction sets less often, it will suffer a performance loss when there are relatively more $Y=1$ events in the data. In this particular case, NAPS appears to perform best relative to class-conditional prediction sets when $\P_{\text{target}}[Y=1]$ is low, but results may vary in other settings where the relationship between the nuisance parameter(s) and labels may be more complex. However, we do not expect NAPS to outperform class-conditional prediction sets (or any method developed for SLS) under SLS-only scenarios.}

\end{document}